\documentclass{article} 
\usepackage{arxiv,times}

\usepackage{natbib}

\usepackage{setspace}

\usepackage{caption}

\usepackage{hyperref}
\usepackage{url}

\usepackage{amssymb}
\usepackage{amsthm}
\usepackage{amstext}
\usepackage{amsmath}

\usepackage{epsfig}
\usepackage{subfigure}

\usepackage{enumitem}
\usepackage{placeins} 
\usepackage{makecell}

\newtheorem{theorem}{Theorem}[section]
\newtheorem{corollary}[theorem]{Corollary}
\newtheorem{lemma}[theorem]{Lemma}
\newtheorem{proposition}[theorem]{Proposition}
\newtheorem{remark}[theorem]{Remark}
\newtheorem{definition}{Definition}
\newtheorem{assumption}{Assumption}

\title{Stochastic Adaptive Gradient Descent Without Descent}

\author{Jean-Francois\ Aujol, J\'er\'emie Bigot \& Camille Castera \thanks{Alphabetical order.} \\
	Univ. Bordeaux\\
	CNRS, Bordeaux INP, IMB, UMR 5251\\
	F-33400 Talence, France\\
	\texttt{\{jean-francois.aujol,jeremie.bigot,camille.castera\}@math.u-bordeaux.fr} \\
}

\newcommand{\NN}{\ensuremath{\mathbb N}}
\newcommand{\N}{{\mathbb N}}
\newcommand{\R}{{\mathbb R}}
\newcommand\ie{\textit{i.e.,}}
\newcommand{\esp}[2][]{\mathbb{E}_{{#1}}\left[#2\right]}
\newcommand{\gf}[1][]{\nabla f_{{#1}}}

\newcommand{\norm}[1]{\left\Vert #1 \right\Vert}
\newcommand{\inner}[2]{\left\langle #1 , #2 \right\rangle}

\newcommand{\xk}{{x_{k}}}
\newcommand{\xkm}{{x_{k-1}}}

\newcommand{\xikm}{{\xi_{k-1}}}

\usepackage{xcolor}
\usepackage{algorithm}
\usepackage{algorithmicx}
\usepackage{algpseudocode}

\begin{document}
	
	\maketitle

	\begin{abstract}
		We introduce a new adaptive step-size strategy for convex optimization with stochastic gradient that exploits the local geometry of the objective function only by means of a first-order stochastic oracle and without any hyper-parameter tuning.
		The method comes from a theoretically-grounded adaptation of the Adaptive Gradient Descent Without Descent method to the stochastic setting.
		We prove the convergence of stochastic gradient descent with our step-size under various assumptions, and we show that it empirically competes against tuned baselines.
	\end{abstract}


	\section{Introduction\label{sec::intro}}
	
	We consider the stochastic convex optimization problem
	\begin{equation}
		\min_{x\in\R^d} f(x) = \min_{x\in\R^d} \esp{f_\xi(x)}, \label{eq:optim}
	\end{equation}
	where the expectation is taken with respect to some random variable $\xi$ 
	such that, $f$ and $f_\xi : \R^d \to \R$ are convex, differentiable and  lower-bounded functions for all $\xi$. We furthermore assume that $f$ has at least one minimizer $x^\star\in \mathbb{R}^d$, and we denote its smallest value by $f^\star=f(x^\star)$.
	When $\xi$ follows a uniform distribution on subsets of $\{1,\ldots,N\}$ (for some integer $N>0$),  \eqref{eq:optim} reads
	\begin{equation}
		\min_{x\in\R^d} f(x) = \min_{x\in\R^d} \frac{1}{N} \sum_{\ell = 1}^{N}  f_{\ell}(x), \label{eq:optim_fullbatch}
	\end{equation}
	where $f_1,\ldots,f_N$ is a set of convex lower-bounded differentiable functions, with locally Lipschitz continuous gradients.
	
	\subsection{Main contributions} \label{sec:main}
	
	We introduce a new adaptive step-size strategy for stochastic gradient descent (SGD) to tackle \eqref{eq:optim}. We build upon the adaptive strategy introduced by \citet{malitsky2019adaptive} for deterministic (full batch) convex optimization.
	They propose a step-size that automatically adapts to the local geometry of $f$ only by means of first-order oracle calls (\ie gradient evaluations). Their method has the notable advantages of not requiring any hyper-parameter tuning nor the global Lipschitz smoothness of the gradient of the objective function (Definition~\ref{def::Lsmooth}). 
	While a stochastic version of the algorithm has been suggested in \cite{malitsky2019adaptive}, it does not inherit the main advantages from the deterministic version: its implementation requires the careful tuning of a hyper-parameter, crucial to guarantee convergence. In contrast, we propose a different stochastic adaptation of their method that preserves the aforementioned advantages in the deterministic setting. As illustrated on Figure~\ref{fig::stepsize_sensitivity}, our algorithm converges without requiring the tuning of any hyper-parameter, under various standard assumptions and models for stochastic optimization.
	
	\subsection{Stochastic Adaptive Descent}
	
	We consider a sequence $(\xi_k)_{k\in\N}$ of independent and identically distributed copies of $\xi$. For any random variable $Z$, $\esp[k-1]{Z}$ denotes the conditional expectation $\esp{Z\mid \mathcal{F}_{k-1}}$, where $\mathcal{F}_{k-1}$ is the filtration up to iteration $k-1$.
	In particular, for all $k\in\N$, we assume access to a stochastic gradient oracle $\gf[\xi_k]$ such that $\forall x\in\R^d$, $\esp[k-1]{\gf[\xi_k](x)} = \gf(x)$.
	We fix $x_0 \in \R^d$ and consider the stochastic recursion
	\begin{equation} \label{eq:algosto}
		x_{k+1} = x_k - \lambda_k \gf[\xi_k](x_k),
	\end{equation}
	where step-size $\lambda_k>0$ is a (possibly random) variable that is assumed to be  independent of $\xi_k$ conditionally on $\mathcal{F}_{k-1}$. These notations are consistent with those of \citet{malitsky2019adaptive} to ease the comparison.
	
	Our main contribution is to prove the benefits of the following adaptive step-sizes $\lambda_k$, for $k\geq 2$:
	\begin{align} \label{eq:choicestep}
		\begin{split}
			&\hat{L}_{k-1} = \frac{\norm{\gf[\xi_{k-1}](x_k) - \gf[\xi_{k-1}](x_{k-1})}}{\norm{x_k-x_{k-1}}},\ \theta_{k-1} = \frac{\lambda_{k-1}}{\lambda_{k-2}} \\
			&\lambda_k  = 
			\left\{
			\begin{array}{@{}lcl@{}}
				\min \left(\frac{1}{2\sqrt{2}\hat{L}_{k-1}}, \lambda_{k-1} \sqrt{1 + \theta_{k-1}} \right) & \text{(V-\textbf{I})} \\
				\min \left( \frac{c_k}{2\sqrt{2}\hat{L}_{k-1}}, \lambda_{k-1} \sqrt{1 + \theta_{k-1}} \right) & \text{(V-\textbf{II})} \\
				\min \left( \frac{c_k}{2\sqrt{2}\hat{L}_{k-1}}, \lambda_{k-1} \sqrt{1 + \left(1-c_k\right)\theta_{k-1}} \right) & \text{(V-\textbf{III})}
			\end{array}
			\right.
		\end{split}
	\end{align}
	with $\lambda_0 > 0$ and $\lambda_1$ defined hereafter.
	We propose three variants of our method, refereed to as V-\textbf{I}, V-\textbf{II} and V-\textbf{III} in \eqref{eq:choicestep}, that only differ by the use of decay $c_k = k^{1/2+\delta}$, for some $0 < \delta < 1/2$.
	This decay provides different convergence properties to the variants. We recommend using V-\textbf{III} which has the strongest theoretical guarantees, but we show that other variants sometimes exhibit good numerical results.   
	The choices in \eqref{eq:choicestep} adapt locally to the geometry of the functions $f_{\xi}$.
	The parameter $\delta>0$ express how fasts the step-sizes decay and is needed for the theoretical convergence analysis of stochastic algorithms.

	\subsection{Making SGD Tuning-free}\label{sec::tuningfree}
	The goal of this paper, and of a long line of work (see related work hereafter), is to set the step-size of SGD in adaptive manners that preserve its rate of convergence while making it as little dependent on hyper-parameter tuning as possible. Ideally it should be ``tuning-free'', \textit{i.e.}, without any hyper-parameter.
	To convey the idea, consider problem~\eqref{eq:optim_fullbatch} and assume (in this section only) that all the functions $f_\ell$ are $\mu$-strongly convex, $L$-smooth (with $0 < \mu \leq L$) and that the variance of the stochastic gradients is bounded by $\sigma^2>0$ (rigorous definitions are deferred to Section~\ref{sec:conv}). Then, \citet{bach:hal-00608041} showed that, for all $k\geq 1$, using the step-sizes $\lambda_k = \frac{\lambda_0}{(k+1)^{1/2+\delta}} $ with $\delta\in(0,1/2)$, the iterates $(x_k)_{k\in\N}$ of vanilla SGD converge to the minimizer $x^\star$ at the following rate:
	\begin{multline}\label{eq::vanillaSCrate}
		\esp{\norm{x_{k+1}-x^\star}^2} \leq 2\exp\left(2L^2\lambda_0^2 \frac{1-(k+1)^{-2\delta}}{\delta}\right)\exp\left(- \frac{\mu\lambda_0}{4}(k+1)^{1/2-\delta}\right)
		\norm{x_0-x^\star}^2 
		+ \frac{4\lambda_0\sigma^2}{\mu (k+1)^{1/2+\delta}}.
	\end{multline}
	Thus the rate of convergence of vanilla SGD is asymptotically of order $\frac{1}{k^{1/2+\delta}}$ and is controlled by the parameter $\delta$. However, the first term in \eqref{eq::vanillaSCrate} may grow \emph{exponentially} with the choice of the initial step-size $\lambda_0$. This can severely impact the non-asymptotic performance of vanilla SGD. Therefore, although $\lambda_0$ has little impact on the asymptotic behavior, its choice is critical in practice, as empirically illustrated for example in \citet{AsiDuchi2019}. 
	This work focuses on proposing a variant of SGD with as little dependence on $\lambda_0$ as possible. We do not study the decay parameter $\delta$ which is another important question.
	
	\citet{malitsky2019adaptive} proposed a deterministic (or full-batch) algorithm with small dependence on $\lambda_0$. However the adaptation they propose in the deterministic setting features a parameter $\alpha>0$ that must be carefully tuned. In particular in the same setting as above, if \textbf{$\alpha \leq \frac{\mu}{2L}$} then they obtain the following upper-bound (with $C_0>0$):
	\begin{equation}\label{eq::MMSCrate}
		\esp{\norm{x_{k+1}-x^\star}^2} \leq C_0\exp\left(-k\frac{\mu\alpha}{L} \right) + \alpha\frac{\sigma^2}{\mu^2},
	\end{equation}
	where the second-term does not vanish in general. Therefore the stochastic adaptation is significantly less ``tuning free'' than the full-batch counterpart as setting $\alpha$ requires knowing $\mu/L$.

	In comparison, in this setting we show that, when choosing $\lambda_k$ according to \eqref{eq:choicestep} V-\textbf{III}, then for all $k\geq 3$:
	\begin{multline}\label{eq::ourRate_strcvx}
		\esp{\norm{x_{k+1}-x^\star}^2}  \leq 2 \exp \left( 
		8 \frac{L^2}{\mu^2} \frac{1- (k-2)^{-2 \delta}}{\delta} \right)
		\exp\left( - \frac{\tau}{16} (k-2)^{1/2-\delta}\right)
		\left(\esp{T_2} + \frac{\sigma^2}{2 L^2} \right) 
		+ \frac{16 \sigma^2}{\tau \mu^2  (k-2)^{1/2 + \delta}}
	\end{multline}
	where $\tau = \min\left\{ \frac{\mu}{2\sqrt{2}L}, \frac{\mu}{\mu+2^\delta\sqrt{2}L}  , \frac{\mu^2}{2^\delta2L^2} \right\}$ is independent of $\lambda_0$, and $\esp{T_2}$ is the expectation at iteration $k=2$ of the Lyapunov function $(T_k)$ we use, defined later in \eqref{eq::Lyap}.
	The asymptotic rates obtained are the same as those in \eqref{eq::vanillaSCrate} for vanilla SGD \emph{without requiring any knowledge of $\mu$ and $L$} nor an extra parameter $\alpha$.
	The two exponential factors do not depend on $\lambda_0$ but rather scale like $\frac{L^2}{\mu^2}$ and $\frac{\mu^2}{L^2}$ respectively (without needing to know this ratio). This is very close to the scaling of \citet{malitsky2019adaptive} detailed in \eqref{eq::MMSCrate}, since in their case $\frac{\alpha \mu}{L}\leq \frac{\mu^2}{2 L^2}$.
	
	Finally, the proposed method (Algorithm~\ref{algo::AdaSGD}) is still not exactly ``parameter-free'' since \eqref{eq::ourRate_strcvx} slightly depends on $\lambda_0$ through $\esp{T_2}$. Yet, we call it ``tuning-free'' since taking a very small value of $\lambda_0$ is always as least as good as large ones, as highlighted in Figure~\ref{fig::stepsize_sensitivity} and also noted by \citet{malitsky2019adaptive} in the deterministic setting.
	We achieve this at the price of one extra stochastic gradient evaluation at each iteration. The same drawback exists with the stochastic approach of \citet{malitsky2019adaptive}, and it is inherent to the idea they propose.
	This additional cost has to be balanced with the cost of hyper-parameter tuning \citep{sivaprasad2020optimizer}.
	We provide additional guarantees beyond the globally Lipschitz continuous and strongly convex settings later in Section~\ref{sec:conv}.
	
	\begin{figure*}[t]
		\centering
		\includegraphics[width=0.8\linewidth]{Figures/sensitivity_legend}
		
		\begin{minipage}{0.47\linewidth}
			\centering
			\includegraphics[width=0.99\linewidth]{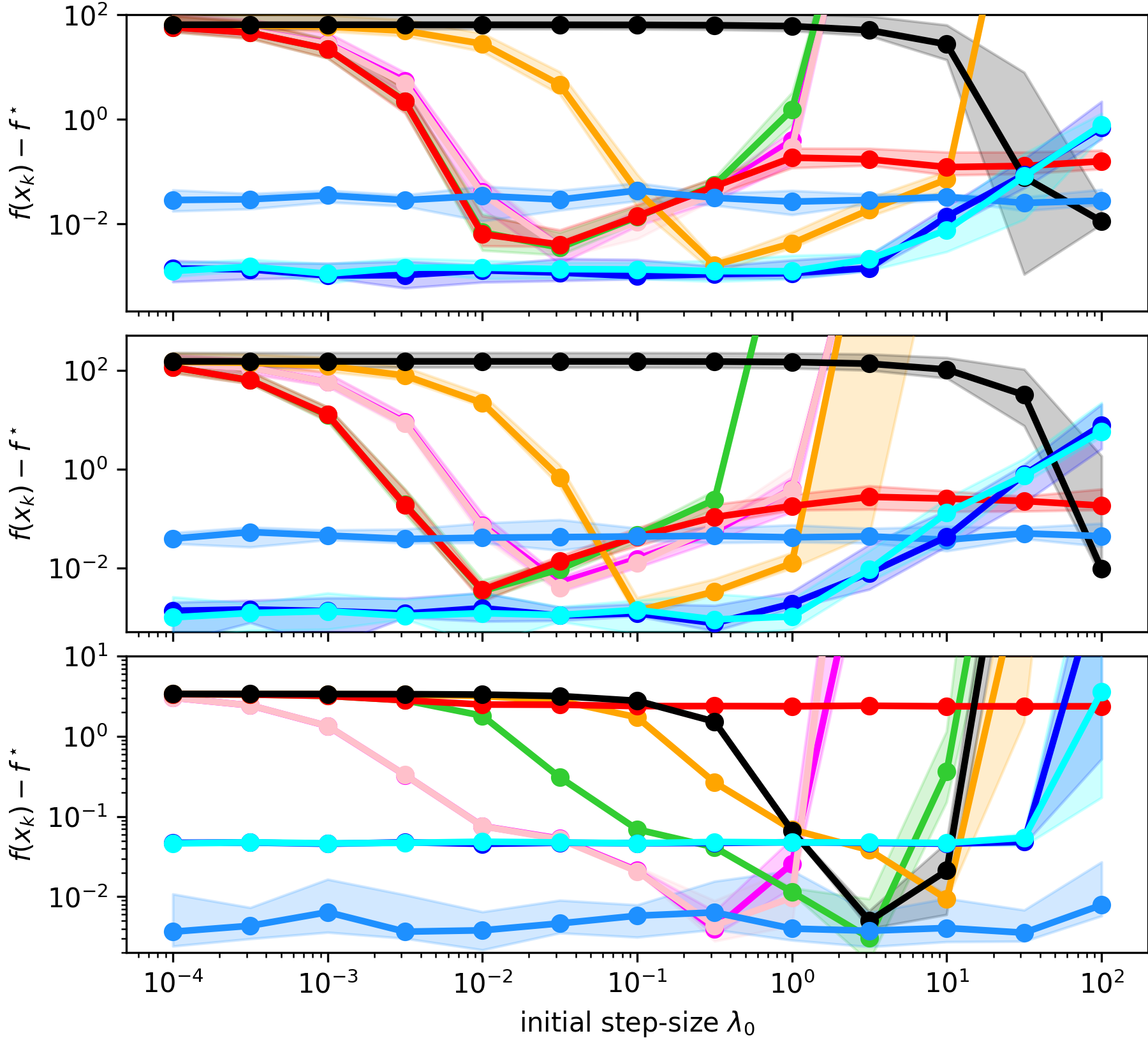}
		\end{minipage}
		\begin{minipage}{0.02\linewidth}
			
		\end{minipage}
		\begin{minipage}{0.49\linewidth}
			\begin{center}
				\vspace{-0.15cm}
				\begin{algorithm}[H]
					\caption{Adaptive SGD without descent\label{algo::AdaSGD}}
					\begin{algorithmic}[1]
						\Require $x_0 \in \R^d$, $\lambda_0 = 10^{-3}$
						\For{$k = 0$ to $\ldots$}
						\State Draw mini-batch $\xi_k$
						\State Evaluate $\gf[\xi_k](x_k)$
						\If{$k \geq 1$}
						\State Evaluate $\gf[\xi_{k-1}](x_k)$
						\If{$k = 1$}
						\State $\lambda_1 = \frac{\norm{x_1 - x_0}}{2\sqrt{2}\norm{\gf[\xi_0](x_1) - \gf[\xi_0](x_0)}}$
						\Else
						\State Compute $\lambda_k$ according to \eqref{eq:choicestep}
						\EndIf
						\EndIf
						\State Store $x_k$, $\lambda_k$, $\lambda_{k-1}$ and $\xi_k$
						\State Update $x_{k+1} = x_k - \lambda_k \gf[\xi_k](x_k)$
						\EndFor
					\end{algorithmic}
				\end{algorithm}
			\end{center}
		\end{minipage}
		\caption{Influence of the initial step-size $\lambda_0$ on the optimality gap $f(x_k)- f^\star$ after $100$ epochs. Each figure represents a different problem (linear, ridge and Poisson regression, see details in Section~\ref{sec:num}).
			Each experiment was ran ten times, solid curve: median value , area: gap between $10\%$ and $90\%$ quantiles.
			Our methods (AdaSGD, the blue curves) always perform well with small values of $\lambda_0$ and thus do not require step-size tuning to compete with optimally-tuned SGD and ``AdaSGD-MM'' from \citet{malitsky2019adaptive}. \label{fig::stepsize_sensitivity}}
	\end{figure*}

	\subsection{Related work}
	
	\noindent
	\textbf{Further adaptation of \citet{malitsky2019adaptive}}
	In the deterministic setting, \citet{malitsky2019adaptive} introduced an adaptive version of GD (here-after called AdaGD), later refined in \citep{malitsky2024adaptive} to improve constant factors in the choice of $\lambda_k$. A proximal version has been developed \citep{latafat2024adaptive}. 
	As far as we know, there is little work on adapting AdaGD to the stochastic case (beyond the heuristics in \citealt{malitsky2019adaptive}). One exception is  due to \citet{defazio2022grad}, who take inspiration from AdaGD, but derives a rather different algorithm \textit{à la} AdaGrad that still features a $\lambda_0$ whose choice is important. 
	
	\noindent
	\textbf{Adaptive methods.}
	Adaptive methods are a very active topic due to their use for training neural networks \citep{duchi2011adaptive, mcmahan2010adaptive, kingma2014adam, tieleman2012lecture}. However, they still feature a step-size parameter (and possibly others), whose choice significantly affects the performances \citep{sivaprasad2020optimizer}.
	Many other strategies exist beyond AdaGD in the deterministic setting ,  \citep{li2023simple,lan2023optimal,khaled2023dowg}, and AdaGD can be seen as a more stable version of the Barzilai-Borwein step-sizes \citep{barzilai1988two, raydan1997barzilai}. Stochastic adaptations of these step-sizes have been proposed for convex \citep{tan2016barzilai} and non-convex optimization \citep{8945980,castera2022second}, but they always mitigate the instability of the methods through additional hyper-parameters.
	Overall, we stress that the stochastic setting possesses additional difficulties that make it significantly more challenging \citep{orabona2020icml}.

	\noindent
	\textbf{Parameter free algorithms and normalized gradients}
	In the stochastic setting, the term \emph{parameter-free} has several meanings.
	It is often connected to online learning where one seeks adaptive strategies of methods that \emph{provably} minimize $f$ as well as the optimally-tuned instance of the method (up to a factor), see \textit{e.g.}, \citep{orabona2014simultaneous, Tuningfree24}. The definition of parameter free may vary depending on the assumptions on the function, the variables of the problems that assumed accessible and the oracles (e.g., accessing $f_\xi$ or not). 
	Despite the name, parameter-free algorithms may still assume additional knowledge of the problem, such as a bound on gradient norms.
	Many parameter-free methods use normalized gradient strategies, \textit{à la} Adagrad: they normalize step-sizes using past (stochastic) evaluations of gradients, which has a provable benefit \citep{pmlr-v89-li19c}. Some of these approaches include \citep{levy2017online,orabona2021parameter,pmlr-v178-faw22a,ivgi2023dog,carmon2022making}.
	When an upper-bound on gradient norms is known, the coin-betting strategy \citep{orabona2016coin, orabona2017training} provides acceleration by estimating the initial optimality gap $\Vert x_0 - x^\star\Vert$.
	We refer to \citet{orabona2023normalized} for a detailed discussion on gradient normalization.   
	When function values are accessible, Polyak step-sizes \citep{polyak1987introduction} can be considered. While originally designed for deterministic optimization, they have been adapted to the stochastic setting \citep{hazan2019revisiting, loizou2021stochastic}, but requires knowledge of the optimal values of the functions $(f_\xi)_\xi$. \citet{chen2019fast} proposed another algorithm using function values.
	If the functions $f$ and the $f_\xi$ are themselves Lipschitz continuous, recent works leverages this to design deterministic and stochastic adaptive strategies \citep{cutkosky2019artificial,defazio2023learning,mishchenko2023prodigy}.

	\noindent
	\textbf{Difference with parameter-free approaches}
	Despite the connections with parameter-free algorithms, our work does not focus on theoretical \emph{speed} comparisons with optimally tuned SGD. Since even in the deterministic setting, \citet{malitsky2019adaptive} did not show any theoretical acceleration for AdaGD compared to GD. The purpose of our work is rather to adapt AdaGD to the stochastic setting while preserving its benefits (see Section \ref{sec::tuningfree}).

	\noindent
	\textbf{Other related approaches}
	Line-search is a common way to alleviate step-size selection but is difficult to adapt to the stochastic setting and often requires auxiliary hyper-parameters for the sake of stability \citep{byrd2012sample, franchini2023learning}, similarly to the Barzilai-Borwein methods.
	Beyond GD, inertial methods are ubiquitous in optimization as they allow achieving optimal rates for convex \citep{nesterov1983method} and strongly convex optimization \citep{nesterov2013introductory}. 
	These methods can also be used with adaptive parameters or adaptive restarting strategies \citep{aujol2024parameter, aujol2025fista} \citep{barre2020complexity, maier2023near}.
	Finally, \textit{universal methods} \citep{nesterov2015universal,li2023simple,grimmer2024optimal} relax the Lipschitz continuity of the gradient by assuming an (unknown) level of Hölder continuity and try to automatically adapt to it. The focus is thus different than that of the parameter-free setting.
	
	\noindent
	\textbf{Convergence of SGD.}
	Despite the different theoretical focus, we do provide theoretical guarantees via the Robbins-Siegmund theorem \citep{RS71}, one of the main tools to prove convergence in the stochastic setting. Other proof strategies are possible depending on the variance assumption of the stochastic algorithm considered, see the discussion in \citet{cortild2025new} for further details, and \citet{garrigos2023handbook} for a recollection of proofs.

	\subsection{Organisation of the paper}
	
	In Section \ref{sec:lyap}, we derive a Lyapunov sequence that is inspired by the one proposed in \citet{malitsky2019adaptive} but with a different adaptation to the stochastic setting, leading to the choice \eqref{eq:choicestep}. In Section \ref{sec:conv} we discuss various assumptions of the functions $f_{\xi}$ that allow to show the convergence of our method. Numerical experiments are reported in Section \ref{sec:num} to illustrate the performances of our approach over a range of stochastic optimization problems. Some auxiliary technical results are postponed the Appendix of the paper.


	\section{Design of the Method} \label{sec:lyap}
	
	The algorithm of \citet{malitsky2019adaptive} is obtained via an original Lyapunov sequence. We present the main ideas leading to it and explain where our approach departs from theirs. Although Lipschitz continuity is not directly used below, we recall it in the Appendix (Definition~\ref{def::Lsmooth}).

	\subsection{Previous Results}\label{sec:prev}
	In the deterministic setting, they propose to use GD (\ie \eqref{eq:algosto} with full batch) algorithm with the step-size $\lambda_k$ defined as
	\begin{equation} \label{eq::lambdak_deterministic}
		\min \left\{ \frac{\norm{x_k-x_{k-1}}}{2\norm{\gf(x_k) - \gf(x_{k-1})}},   \lambda_{k-1} \sqrt{ 1 + \theta_{k-1} }   \right\},
	\end{equation}
	where $\theta_{k-1} =  \frac{ \lambda_{k-1} }{\lambda_{k-2}}$.
	To ease the reading, we use the same notation $(x_k)_{k\in\N}$ to denote the iterates both in stochastic and deterministic settings. Their main idea is that for the step-size \eqref{eq::lambdak_deterministic}, the sequence defined for $k\geq 2$ by
	\begin{equation}\label{eq::LyapTk}
		T_k =  \norm{x_{k+1}-x^\star}^2 + 2  \left(  \lambda_{k} \left( 1 + \theta_{k} \right) \right) \left(f(x_{k})-f^\star\right) 
		+   \frac{\norm{x_{k+1}-x_{k}}^2}{2}
	\end{equation}
	is non-increasing along the iterates of GD. We thus call $(T_k)_{k\in\N}$ a Lyapunov sequence. This is key as the choice \eqref{eq::lambdak_deterministic}, does not guarantee the decay of $(f(x_k))_{k\in\N}$, which is the usual Lyapunov sequence used for GD. This motivated the name ``adaptive gradient descent without descent'' (AdaGD).
	Not ensuring descent at each iteration allows for a choice $\lambda_k$ that adapts to the local Lipschitz constant of $\nabla f$ as argued in   \citet{malitsky2019adaptive}.

	In the case of SGD, the authors consider the choice
	$
	\lambda_k  = \min \left\{ \alpha \Lambda_k,   \lambda_{k-1} \sqrt{ 1 + \theta_{k-1} } \right\}
	$
	with either
	$
	\Lambda_k =  \frac{\norm{x_k-x_{k-1}}}{\norm{\gf[\xi_k](x_k) - \gf[\xi_k](x_{k-1})}}$
	or
	$
	\Lambda_k =  \frac{\norm{x_k-x_{k-1}}}{\norm{\gf[\zeta_k](x_k) - \gf[\zeta_k](x_{k-1})}}$ with $\zeta_k$ an independent copy of $\xi_k$,
	and where $\alpha > 0$ is a hyper-parameter \emph{to be tuned}. 
	They call these choices ``biased'' and ``unbiased'' for a reason we explain later.
	In contrast, in \eqref{eq:choicestep}, we evaluate the difference of gradients on the \emph{previous} random variable $\xi_{k-1}$: $\norm{\gf[\xi_{k-1}](x_k) - \gf[\xi_{k-1}](x_{k-1})}$. This small difference is crucial as it allows getting rid of the hyper-parameter $\alpha$ compared to their approach, and whose choice depends on the (unknown) smoothness constants of the functions $\gf[\xi_k]$.

	\subsection{Adaptive stochastic gradient descent}
	
	We now derive the Lyapunov sequence and highlight the main differences with compared \citet{malitsky2019adaptive}. 
	In what follows, we assume \emph{and ensure} that the step-size satisfies the following.
	
	\begin{assumption} \label{assump_inde}
		The step-size $\lambda_k>0$ is independent of $\xi_k$ conditionally on $\mathcal{F}_{k-1}$. 
	\end{assumption}
	\begin{remark}
		At this stage Assumption \ref{assump_inde} is required for the computations below, but we stress that our step-sizes \eqref{eq:choicestep} are designed \emph{so that the assumption always holds true}.     
	\end{remark}
	Let $k\geq 2$, 	we use notation $\Delta_k = x_{k+1} - x_k$, as in the deterministic case we begin with the decomposition
	\begin{alignat*}{1}
		\norm{x_{k+1} - x^\star}^2 
		=& \norm{\Delta_k}^2 + 2 \inner{\Delta_k}{x_k-x^\star} 
		+ \norm{x_k-x^\star}^2
		\\
		=& \norm{\Delta_k}^2 - 2 \inner{\lambda_k\gf[\xi_{k}](x_k)}{x_k-x^\star}
		+ \norm{x_k-x^\star}^2.
	\end{alignat*}

	Taking conditional expectations on both sides, we obtain
	\begin{equation}\label{eq:start}
		\esp[k-1]{\norm{x_{k+1} - x^\star} ^2} = \esp[k-1]{\norm{\Delta_k}^2}
		- 2 \inner{\esp[k-1]{\lambda_k\gf[\xi_{k}](x_k)}}{x_k-x^\star}
		+  \norm{x_k-x^\star}^2. 
	\end{equation}
	
	Focusing on the second term in the right-hand side: due to Assumption~\ref{assump_inde} the conditional independence between $\xi_k$ and $\lambda_k$ allows splitting the conditional expectation into a product. This is the reason for the terminology ``unbiased'' as it then allows to use
	the property that $\esp[k-1]{\gf[\xi_k](x_k)} = \gf(x_k)$, meaning that $\gf[\xi_k](x_k)$ is an unbiased estimator of $\gf(x_k)$. We obtain
	\begin{align}\label{convexe}
		&\inner{\esp[k-1]{\lambda_k\gf[\xi_{k}](x_k)}}{x_k-x^\star} = \esp[k-1]{\lambda_k}\inner{\gf(x_k)}{x_k-x^\star}
		\geq  \esp[k-1]{\lambda_k}\left(f(x_k)-f^\star\right), 
	\end{align}
	where we used the convexity of $f$ in the last inequality. The function values $f(x_k)$ appear, so we expect the quantity $f(x_{k-1})$ to appear as well later in our Lyapunov analysis.

	Now focusing on $\Delta_k$, it is trivial but convenient to write
	$\norm{\Delta_k}^2 = 2\norm{\Delta_k}^2 - \norm{\Delta_k}^2$ and use the identity
	$2\norm{\Delta_k}^2 = -2\inner{\lambda_k\gf[\xi_k](x_k)}{\Delta_k}$.
	We then make an extra stochastic gradient appear:
	\begin{align}\label{eq::diffxk}
		\begin{split}
			-2&\inner{\lambda_k\gf[\xi_k](x_k)}{\Delta_k} = \underbrace{-2\lambda_k\inner{\gf[\xi_{k-1}](x_{k-1})}{\Delta_k}}_{
				:=B_k}
			\underbrace{-2\lambda_k\inner{\gf[\xi_k](x_k) - \gf[\xi_{k-1}](x_{k-1})}{\Delta_k}}_{
				:=A_k}
		\end{split}
	\end{align}
	This is a first pivoting choice in our analysis, it yields two terms that we analyze separately.
	
	\subsubsection{Bounding \texorpdfstring{$B_k$}{Bk}}
	
	We first use \eqref{eq:algosto} to substitute $\gf[\xi_{k-1}](x_{k-1})$ in $B_k$:
	\begin{equation*}
		B_k = -2\lambda_k \inner{-\frac{1}{\lambda_{k-1}}(x_k-x_{k-1})}{-\lambda_k \gf[\xi_k](x_k)} 
		=  -2 \frac{\lambda_k^2}{\lambda_{k-1}} \inner{x_k-x_{k-1}}{\gf[\xi_k](x_k)}
	\end{equation*} 
	We take conditional expectation and use again Assumption~\ref{assump_inde} and the convexity of $f$ to get
	\begin{align} 
		\nonumber
		\esp[k-1]{B_k} & = 	-2 \frac{\esp[k-1]{\lambda_k^2}}{\lambda_{k-1}} \inner{x_k-x_{k-1}}{\gf(x_k)} 
		\\ \label{eq:(b)}
		&\leq 2 \frac{\esp[k-1]{\lambda_k^2}}{\lambda_{k-1}} (f(x_{k-1})-f(x_k)) =
		2 \frac{\esp[k-1]{\lambda_k^2}}{\lambda_{k-1}} (f(x_{k-1}) - f^\star) 
		- 2 \frac{\esp[k-1]{\lambda_k^2}}{\lambda_{k-1}}(f(x_{k}) - f^\star) 
	\end{align}
	We obtain a term with $f(x_{k-1})$ that will be balanced by \eqref{convexe} and will yield a condition between $\lambda_k$ and $\lambda_{k-1}$, similar to that in \eqref{eq::lambdak_deterministic}. Our choice of plugging $\gf[\xi_{k-1}](x_{k-1})$ into \eqref{eq::diffxk} allows our analysis to remain faithful to the approach followed in the deterministic case.

	\subsubsection{Dealing with the term \texorpdfstring{$A_k$}{Ak}}
	
	We use $\vert\inner{x}{y}\vert\leq \norm{x}^2 + \frac{\norm{y}^2}{4}$, $\forall x,y\in\R^d$ on $A_k$ in \eqref{eq::diffxk}:
	\begin{equation*}
		-2\lambda_k\inner{\gf[\xi_k](x_k) - \gf[\xi_{k-1}](x_{k-1})}{\Delta_k} 
		\leq
		2\lambda_k^2\norm{\gf[\xi_k](x_k) - \gf[\xi_{k-1}](x_{k-1})}^2 +  \frac{\norm{\Delta_k}^2}{2}.
	\end{equation*}
	The second term in the right-hand side above will be compensated by the $- \norm{\Delta_k}^2$ obtained above \eqref{eq::diffxk}. The first term is more involved. We plug again an extra stochastic gradient:
	\begin{align}\label{eq:key}
		\begin{split}
			\norm{\gf[\xi_k](x_k) - \gf[\xi_{k-1}](x_{k-1})}^2
			=&\left\Vert\gf[\xi_k](x_k) - \gf[\xi_{k-1}](x_{k})\right.
			\left.+ \gf[\xi_{k-1}](x_k) - \gf[\xi_{k-1}](x_{k-1})\right\Vert^2 
			\\
			\leq&  2\underbrace{\norm{\gf[\xi_k](x_k) - \gf[\xi_{k-1}](x_{k})}^2}_{\text{induced by sampling}} 
			+ 2\underbrace{\norm{\gf[\xi_{k-1}](x_k) - \gf[\xi_{k-1}](x_{k-1})}^2}_{\text{induced by curvature}},
		\end{split}
	\end{align}
	where we used $\norm{x+y}^2\leq 2 \norm{x}^2 + 2\norm{y}^2$.  The decomposition \eqref{eq:key} is a subtle but crucial difference with \citet{malitsky2019adaptive} in the stochastic setting. As highlighted in \eqref{eq:key} it allows bounding the difference of stochastic gradients by a sampling error and a ``curvature-induced'' term. In the deterministic setting ($\xi_k=\xi_{k-1}$) the first term vanishes and the curvature term is exactly the difference of gradients in \eqref{eq::lambdak_deterministic}. This leads to our first condition on $\lambda_k$ in \eqref{eq:choicestep}:	
	\begin{equation}\label{eq::firstcondition}
		\lambda_k \leq 
		\frac{\norm{x_k-x_{k-1}}}{2\sqrt{2}\norm{\gf[\xi_{k-1}](x_k) - \gf[\xi_{k-1}](x_{k-1})}}
	\end{equation}
	that allows bounding the curvature term:
	$ 4\lambda_k^2{\norm{\gf[\xi_{k-1}](x_k) - \gf[\xi_{k-1}](x_{k-1})}^2} 
	\leq 
	\frac{\norm{\Delta_{k-1}}^2}{2}.
	$
	Therefore, using again Assumption~\ref{assump_inde}, we  obtain 
	\begin{equation}
		\esp[k-1]{A_k} \leq
		\esp[k-1]{\lambda_k^2}\esp[k-1]{\norm{\gf[\xi_k](x_k) - \gf[\xi_{k-1}](x_k)}^2} + \frac{\norm{\Delta_{k-1}}^2}{2} 
		+ \frac{\norm{\Delta_k}^2}{2}.  \label{eq:(a)}
	\end{equation}
	
	\subsection{Decay of the Lyapunov sequence in Expectation}\label{sec:first_result}
	We use \eqref{eq:(b)} and \eqref{eq:(a)} in \eqref{eq::diffxk} to bound the first term in \eqref{eq:start}:
	\begin{align*}
		\esp[k-1]{\norm{\Delta_{k}}^2}  
		&\leq  \esp[k-1]{\lambda_k^2}\esp[k-1]{\norm{\gf[\xi_k](x_k) - \gf[\xi_{k-1}](x_k)}^2} 
		+ \frac{\norm{\Delta_{k-1}}^2}{2}
		- \esp[k-1]{\frac{\norm{\Delta_{k}}^2}{2}} +
		\\&
		2 \frac{\esp[k-1]{\lambda_k^2}}{\lambda_{k-1}} (f(x_{k-1}) - f^\star) 
		- 2 \frac{\esp[k-1]{\lambda_k^2}}{\lambda_{k-1}}(f(x_{k}) - f^\star) ,
	\end{align*}
	which, going back to \eqref{eq:start} and rearranging, leads to
	\begin{multline*}{1}
		\esp[k-1]{\norm{x_{k+1} - x^\star}^2} + \esp[k-1]{\frac{\norm{\Delta_{k}}^2}{2}}
		+ 2\left(\esp[k-1]{\lambda_k}+\frac{\esp[k-1]{\lambda_k^2}}{\lambda_{k-1}}\right)\left(f(x_k)-f^\star\right) 
		\\
		\leq
		\norm{x_k-x^\star}^2  
		+ \frac{\norm{\Delta_{k-1}}^2}{2} 
		+ 2 \frac{\esp[k-1]{\lambda_k^2}}{\lambda_{k-1}} (f(x_{k-1})-f(x_k))
		+ 4\esp[k-1]{\lambda_k^2}\esp[k-1]{\norm{\gf[\xi_k](x_k) - \gf[\xi_{k-1}](x_k)}^2} .
	\end{multline*}
	
	We have almost recovered an inequality on the sequence $(T_k)_{k\in\N}$ defined in \eqref{eq::LyapTk}. We want to make $(T_k)_{k\in\N}$ non-increasing, up to the last term above that we will handle later. This is where our second condition comes in. Recall $\theta_{k-1} =  \lambda_{k-1} / \lambda_{k-2}$ and assume that for all $k\geq 2$,
	\begin{equation}\label{eq::cond2strat3}
		\esp[k-1]{\lambda_{k}^2}
		\leq \lambda_{k-1}^2 \left( 1 + \theta_{k-1} \right)
	\end{equation}
	then, under this condition and \eqref{eq::firstcondition}, we just obtained a new stochastic adaptation of AdaGD.
	
	\begin{proposition}  \label{prop:lyap}Under Assumption \ref{assump_inde} and conditions \eqref{eq::firstcondition}
		and \eqref{eq::cond2strat3}, $T_k$ defined in \eqref{eq::LyapTk} satisfies
		\begin{equation} \label{eq::Lyap}
			\esp[k-1]{T_{k}}  \leq 
			T_{k-1} 
			+ 4  \esp[k-1]{\lambda_k^2}\esp[k-1]{\norm{\gf[\xi_k](x_k) - \gf[\xi_{k-1}](x_k)}^2}.
		\end{equation}
	\end{proposition}
	
	\begin{remark}
		Condition \eqref{eq::cond2strat3} involves an expectation, which is inconvenient in practice but we simplify this by taking the stricter requirement $\lambda_k \leq \lambda_{k-1} \sqrt{ 1 + \theta_{k-1} } $, in \eqref{eq:choicestep}.
	\end{remark}

	In particular for the three variants that we propose in \eqref{eq:choicestep}, conditions \eqref{eq::firstcondition} and \eqref{eq::cond2strat3} are fulfilled. 
	\begin{corollary}\label{col::Lyap}
		For any of the three choices in \eqref{eq:choicestep}, Assumption~\ref{assump_inde} and Proposition~\ref{prop:lyap} both hold true for Algorithm~\ref{algo::AdaSGD}.
	\end{corollary}

	To conclude this section, we importantly note that our analysis unifies the deterministic and stochastic settings. Indeed, in the deterministic setting, $\xi_k=\xi_{k-1}$, so \eqref{eq::Lyap} simply reads $T_k \leq T_{k-1}$ and conditions \eqref{eq::firstcondition} and \eqref{eq::cond2strat3} are exactly \eqref{eq::lambdak_deterministic}, up to a $\sqrt{2}$ factor. 
	The main difference in the stochastic case is $ \esp[k-1]{\lambda_k^2}\esp[k-1]{\norm{\gf[\xi_k](x_k) - \gf[\xi_{k-1}](x_k)}^2}$ in \eqref{eq::Lyap}, which can informally be interpreted as the variance of the noise induced by stochastic approximations. 
	A classical challenge in stochastic optimization amounts to controlling this term to derive convergence results, which we now do.


	\section{Convergence Analysis} \label{sec:conv}
	
	Throughout this section, we consider Algorithm~\ref{algo::AdaSGD} with any of the three choices \eqref{eq:choicestep}. Remark that with such choices, $\lambda_k$ is not a random variable conditionally on $\mathcal{F}_{k-1}$, so we simply have $\esp[k-1]{\lambda_k^2} = \lambda_k^2$ and as stated, Assumption~\ref{assump_inde} holds true.
	We derive the convergence of Algorithm~\ref{algo::AdaSGD}, and defer the proofs of this Section to the appendix.
	According to Corollary~\ref{col::Lyap}, the sequence $(T_k)_{k\in\N}$ satisfies \eqref{eq::Lyap} for Algorithm~\ref{algo::AdaSGD} in the stochastic setting, but  \eqref{eq::Lyap} features an additional term discussed above that we want to bound. 
	We first show in Lemma~\ref{lemma_bound} of Appendix~\ref{sec::varTransfer} that $\forall k\geq 1$ that defining
	$$
	\sigma_k^2 = \esp[k-1]{\norm{\gf[\xi_k](x^\star)}^2} +  \esp[k-1]{\norm{\gf[\xi_{k-1}](x^\star)}^2},
	$$
	Then
	\begin{align}
		\nonumber
		\esp[k-1]{\norm{\gf[\xi_k](x_k) - \gf[\xi_{k-1}](x_k)}^2} 
		\leq & 4 \esp[k-1]{\norm{\gf[\xi_k](x_k) - \gf[\xi_k](x^\star)}^2} 
		\\ \label{eq::transfer}
		&+ 4 \esp[k-1]{\norm{\gf[\xi_{k-1}](x_k) - \gf[\xi_{k-1}](x^\star)}^2} + 4 \sigma_k^2.
	\end{align}	
	The first two terms in \eqref{eq::transfer} may vanish asymptotically if $x_k$ converges to $x^\star$, but
	$\sigma_k$ is bounded away from zero in general.\footnote{Except in the so-called interpolated case where $x^\star$ minimizes all the $f_\xi$'s.}
	This evidences that ensuring convergence in the stochastic setting requires additional assumptions on the objective function.
	Remark that so far the results hold without using the (local) Lipschitz continuity of the $\gf[\xi]$'s.	We present three possible sets of assumptions that allow controlling \eqref{eq::transfer}.

	\begin{assumption}\label{ass:VarianceControl} One of the following assumptions holds true.
		\begin{enumerate}[label=(\ref{ass:VarianceControl}-{\roman*}),leftmargin=*]
			\item The function $f$ has the finite-sum structure \eqref{eq:optim_fullbatch} and each $\gf[\ell]$ is $L_\ell$-Lipschitz continuous. \label{ass:VCi}
			\item There exists $L>0$ such that each $\gf[\xi]$ is $L$-Lipschitz and $\sigma>0$ such that $\forall k, \sigma_k^2<\sigma^2$.\label{ass:VCii}
			\item The function $f$ has the finite-sum structure \eqref{eq:optim_fullbatch} and the iterates $(x_k)_{k\in\N}$ are bounded.\label{ass:VCiii}
		\end{enumerate}
	\end{assumption}
	Any of these sets of properties provides a uniform control on the problematic term in \eqref{eq::Lyap}. Assumption~\ref{ass:VCii} is necessary when one minimizes infinite sums, Assumption~\ref{ass:VCi} requires global Lipschitz continuity for finite sums, which can be relaxed in \ref{ass:VCiii} but at the cost of assuming boundedness of the iterates, which can only be checked \textit{a posteriori}.
	\begin{lemma}\label{lem::uniformControl}
		Under the conditions in Assumption~\ref{ass:VarianceControl}, $\forall k \geq 1$, there exists $L,\sigma>0$ such that
		\begin{equation*}
			\lambda_k^2\esp[k-1]{\norm{\gf[\xi_k](x_k) - \gf[\xi_{k-1}](x_k)}^2} 
			\leq 8 \lambda_k^2 L^2 \norm{x_k - x^\star}^2 + 4\lambda_k^2\sigma^2
		\end{equation*}
	\end{lemma}
	One can see that when $\lambda_k^2$ converges to zero ``fast enough'', we should obtain convergence of $(T_k)_{k\in\N}$. We now present the main convergence results of the paper in the next two sections.

	\subsection{Almost sure convergence} \label{sec:as}
	
	We begin with a general theorem.
	
	\begin{theorem}\label{thm:convergence} Consider Algorithm~\ref{algo::AdaSGD} with step-size \eqref{eq:choicestep}. If Assumption~\ref{ass:VarianceControl} holds true,
		then whenever the sequence of step-sizes $(\lambda_{k})_{k \in \N}$ satisfies
		$	\sum_{k \geq 0} \lambda_{k}^2 < +\infty \mbox{ almost surely,}$
		the Lyapunov sequence $(T_k)_{k\in\N}$ converges almost surely.
	\end{theorem}
	The proof relies on the Robbins-Siegmund theorem \citep{RS71} that is recalled in the Appendix for completeness. 
	Theorem~\ref{thm:convergence} requires the sequence of step-sizes to be square-summable. 
	This is true in the following settings (as proved in Appendix~\ref{sec:model_App}), and possibly others:
	\begin{enumerate}[label=\textbf{Case}-{\arabic*}, ,leftmargin=*]
		\item when there exists $\mu>0$ such that each $f_\xi$ is $\mu$-strongly convex; \label{case::strcvx}
		\item in least-square regression: $f$ is a finite sum \eqref{eq:optim_fullbatch} and $\forall x\in\R^d, f_\xi(x) = (\inner{w_\xi}{x} - y_\xi)^2$, for $w_\xi\in\R^d$ and $y_\xi\in\R$; \label{case::linreg}
		\item more generally when $f$ is a finite sum \eqref{eq:optim_fullbatch} of \emph{ridge} functions: $\forall x\in\R^d, f_\xi(x) = g_\xi(\inner{w_\xi}{x})$, and $g_\xi:\R\to\R$ is strongly convex. \label{case::ridge}
	\end{enumerate}

	We recall, that $f_\xi$ is $\mu$-strongly	convex means $f_\xi - \frac{\mu}{2} \norm{\cdot}^2$ is convex.
	\ref{case::ridge} is a generalization of linear regression that has important applications in the literature. See Section~\ref{sec:model_App} of the appendix for further discussion on these three cases.
	In these settings we obtain the convergence of the algorithm.
	\begin{corollary}
		Consider Algorithm~\ref{algo::AdaSGD} with step-size V-\textbf{II} or \textbf{III} from \eqref{eq:choicestep}. 
		Under Assumption~\ref{ass:VarianceControl}, the sequence $(T_k)_{k\in_N}$ defined in \eqref{eq::Lyap} converges almost surely for any of \ref{case::strcvx}, \ref{case::linreg} and~\ref{case::ridge}.
	\end{corollary}

	We just provided prominent practical cases in which Algorithm~\ref{algo::AdaSGD} converges without requiring any hyper-parameter tuning. We now show that it is actually possible to improve the results and derive rates of convergence in \ref{case::strcvx}.

	\begin{figure*}[t]
		\centering
		\begin{tabular}{ccc}
			\multicolumn{3}{c}{\includegraphics[width=0.8\textwidth]{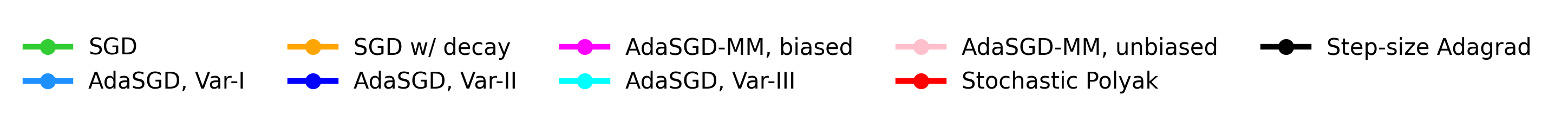}}
			\\
			\makecell{
				\footnotesize Linear -- Diabetes\\
				\includegraphics[width=0.25\textwidth]{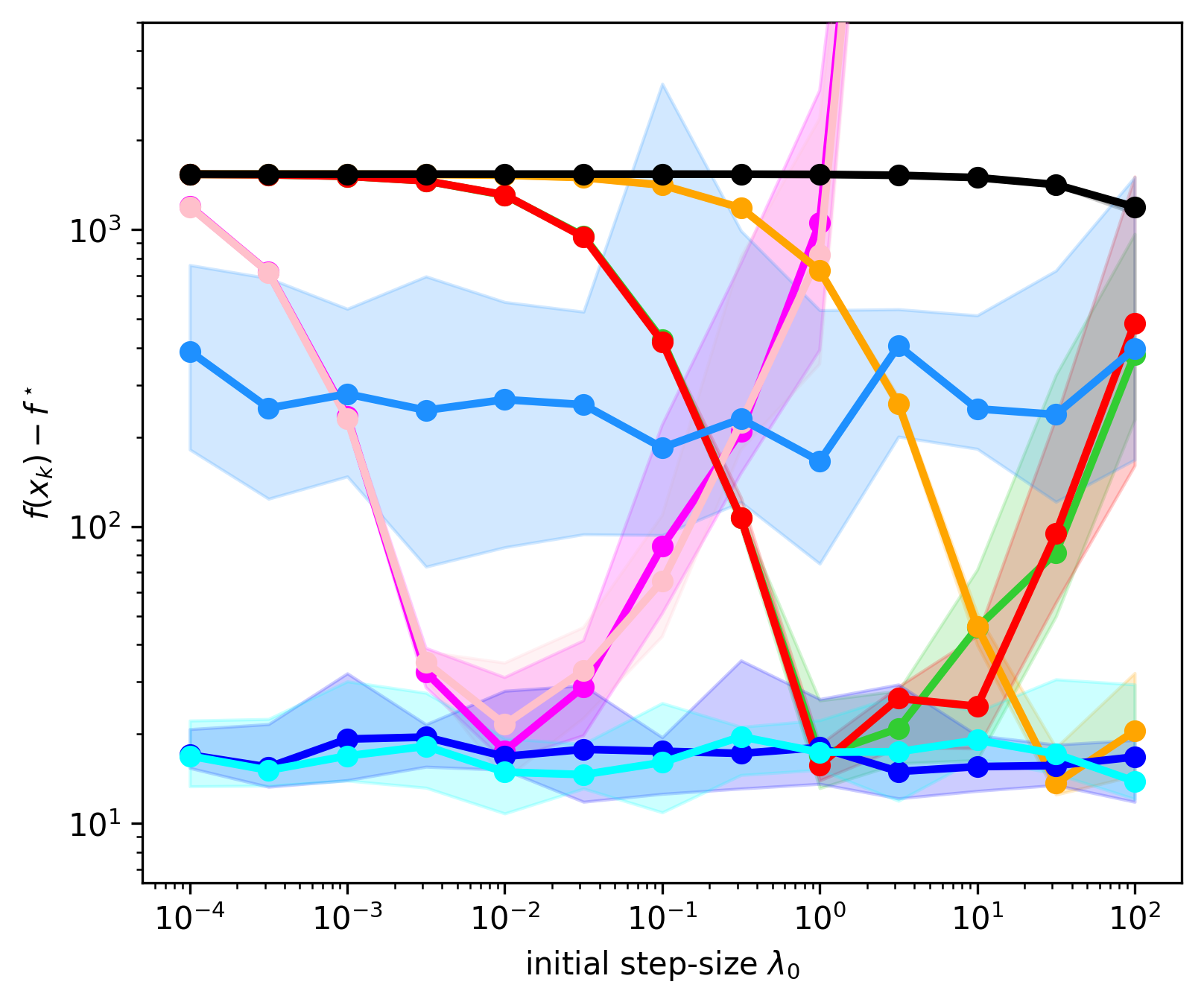} 
			}&
			\makecell{
				\footnotesize Logistic -- 2moons\\
				\includegraphics[width=0.25\textwidth]{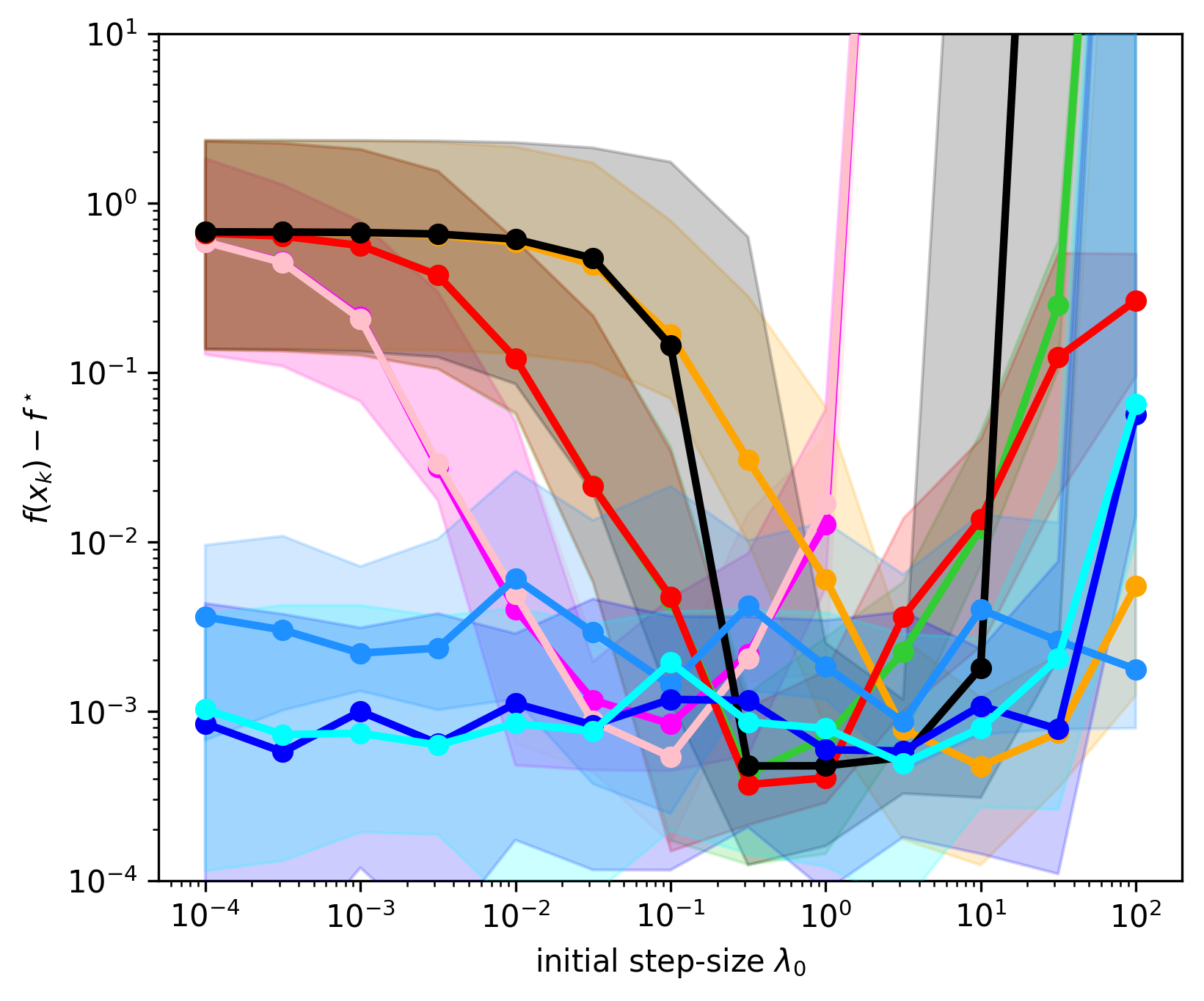}
			}
			&
			\makecell{
				\footnotesize Logistic -- w8a\\
				\includegraphics[width=0.25\textwidth]{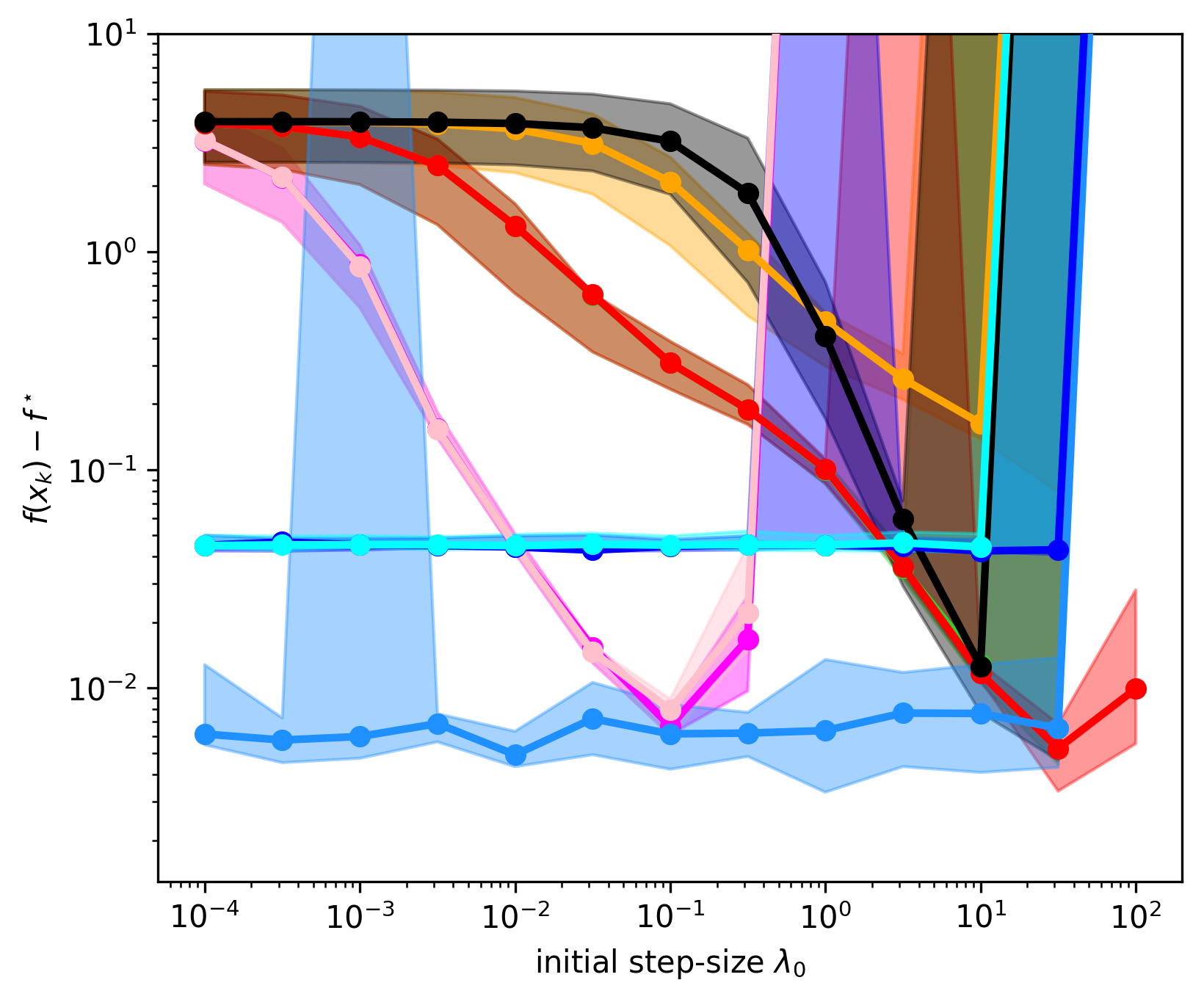}
			}
		\end{tabular}
		\caption{Experiments in the same setting as in Figure~\ref{fig::stepsize_sensitivity} on three other optimization problems. Except for a few outliers, out methods are significantly more stable to the choice of $\lambda_0$ than other algorithms and always work well for small values of $\lambda_0$.
		} \label{fig::detailed_sensitivity}
	\end{figure*}

	\subsection{Convergence rates in expectation}
	When the functions $f_\xi's$ are $\mu$-strongly convex (\ref{case::strcvx}), V-\textbf{III} of step-size \eqref{eq:choicestep} provides the algorithm with a rate of convergence.

	\begin{theorem}\label{thm::CVrates} Consider \ref{case::strcvx} with Assumption~\ref{ass:VarianceControl}. Then for V-\textbf{III} of the step-size \eqref{eq:choicestep} with any $0 < \delta < 1/2$, {\color{black} one has that, for any $k \geq 3$,
			\begin{equation}\label{eq:convrate}
				\esp{T_k}   \leq  2\exp \left( 
				8 \frac{L^2}{\mu^2} \frac{1- (k-2)^{-2 \delta}}{\delta} - \frac{\tau}{16} (k-2)^{1/2-\delta} \right)\left( \esp{T_2} + \frac{\sigma^2}{2 L^2} \right)
				+ \frac{16 \sigma^2}{\tau \mu^2  } (k-2)^{-(1/2 + \delta)},
			\end{equation}
			where $\tau = \min\left\{ \frac{\mu}{2\sqrt{2}L}, \frac{\mu}{\mu+2^\delta\sqrt{2}L}  , \frac{\mu^2}{2^\delta2L^2} \right\}$.
		}
	\end{theorem}
	
	The first term on the right-hand side of \eqref{eq:convrate} is the only one that slightly depends on $\lambda_0$ through $\esp{T_2}$ and it vanishes at the exponential rate $\exp(-(k-2)^{1/2-\delta})$. The asymptotic behavior is thus controlled by the second term and the rate is $\mathcal{O}\left( \frac{1}{k^{1/2 + \delta}} \right)$.
	Moreover, since $\esp{\norm{x_{k+1}-x^\star}^2} \leq \esp{T_k}$,  the  same convergence rate holds for $\esp{\norm{x_{k+1}-x^\star}^2}$ stated in \eqref{eq::ourRate_strcvx}.
	This is the same asymptotic rate as for vanilla SGD and the method of \citet{malitsky2019adaptive} but with much weaker dependence on the parameter $\lambda_0$ as explained in Section~\ref{sec::tuningfree}.
	We recommend using V-\textbf{III} of the step-size  as it provides the strongest guarantees.


	\section{Numerical experiments} \label{sec:num}

	\begin{figure*}[t]
		\centering
		\begin{tabular}{ccc}
			\multicolumn{3}{c}{\includegraphics[width=0.8\textwidth]{Figures/sensitivity_legend.png}}
			\\
			\makecell{
				\footnotesize Linear -- Synthetic\\
				\includegraphics[width=0.25\textwidth]{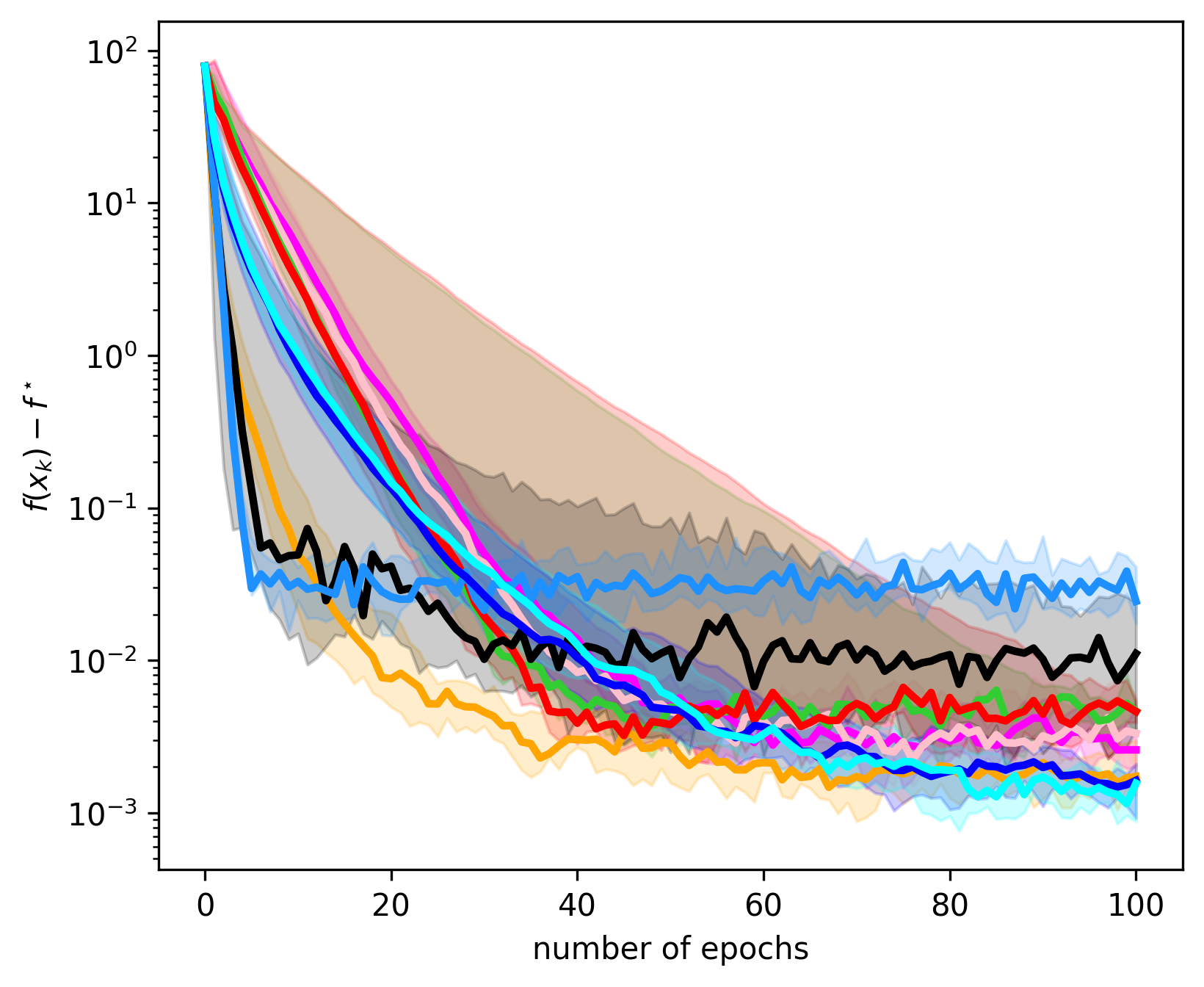}
			}&
			\makecell{
				\footnotesize Linear -- Diabetes\\
				\includegraphics[width=0.25\textwidth]{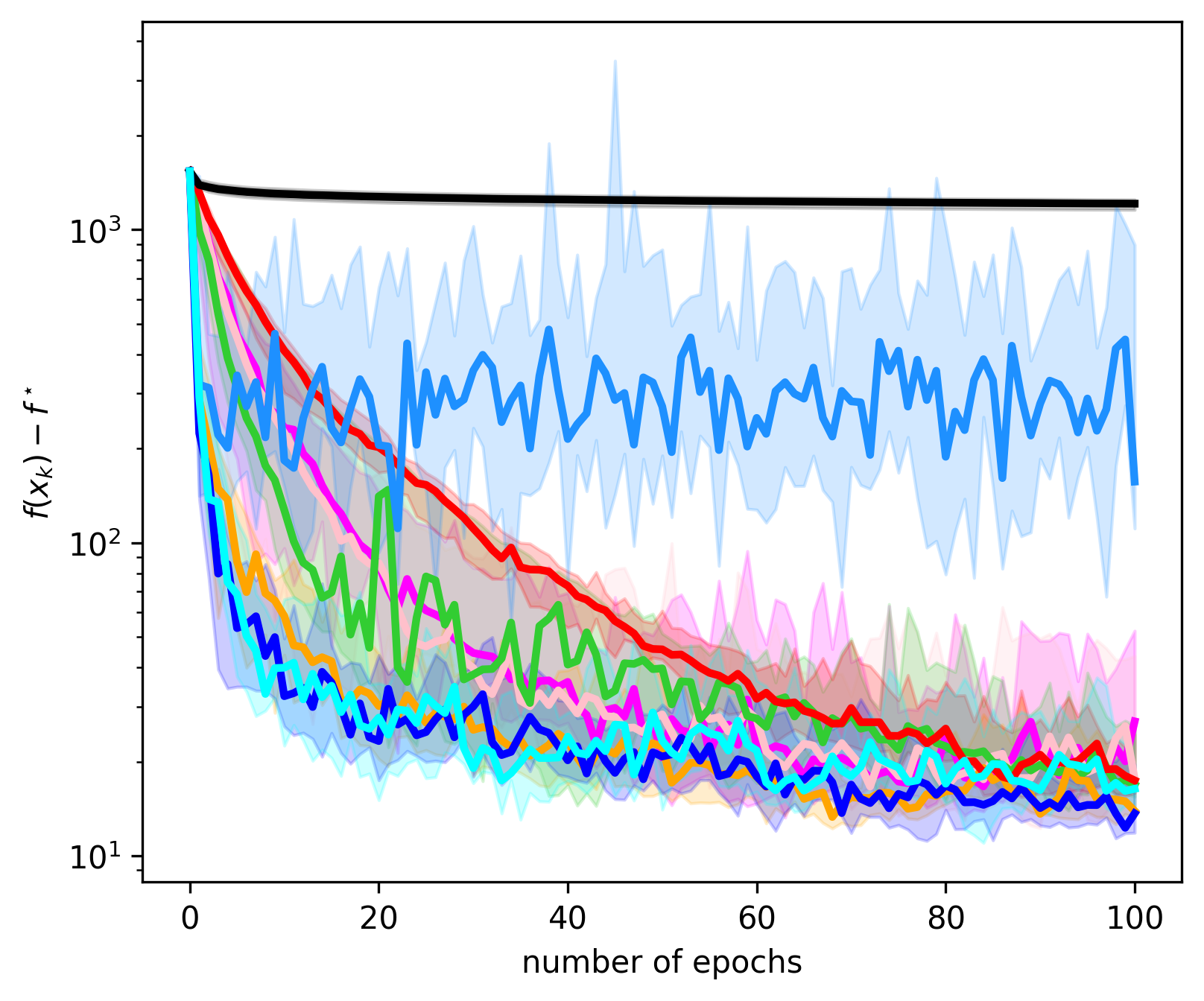} 
			}&
			\makecell{
				\footnotesize Sum-of-Ridges -- Synthetic\\
				\includegraphics[width=0.25\textwidth]{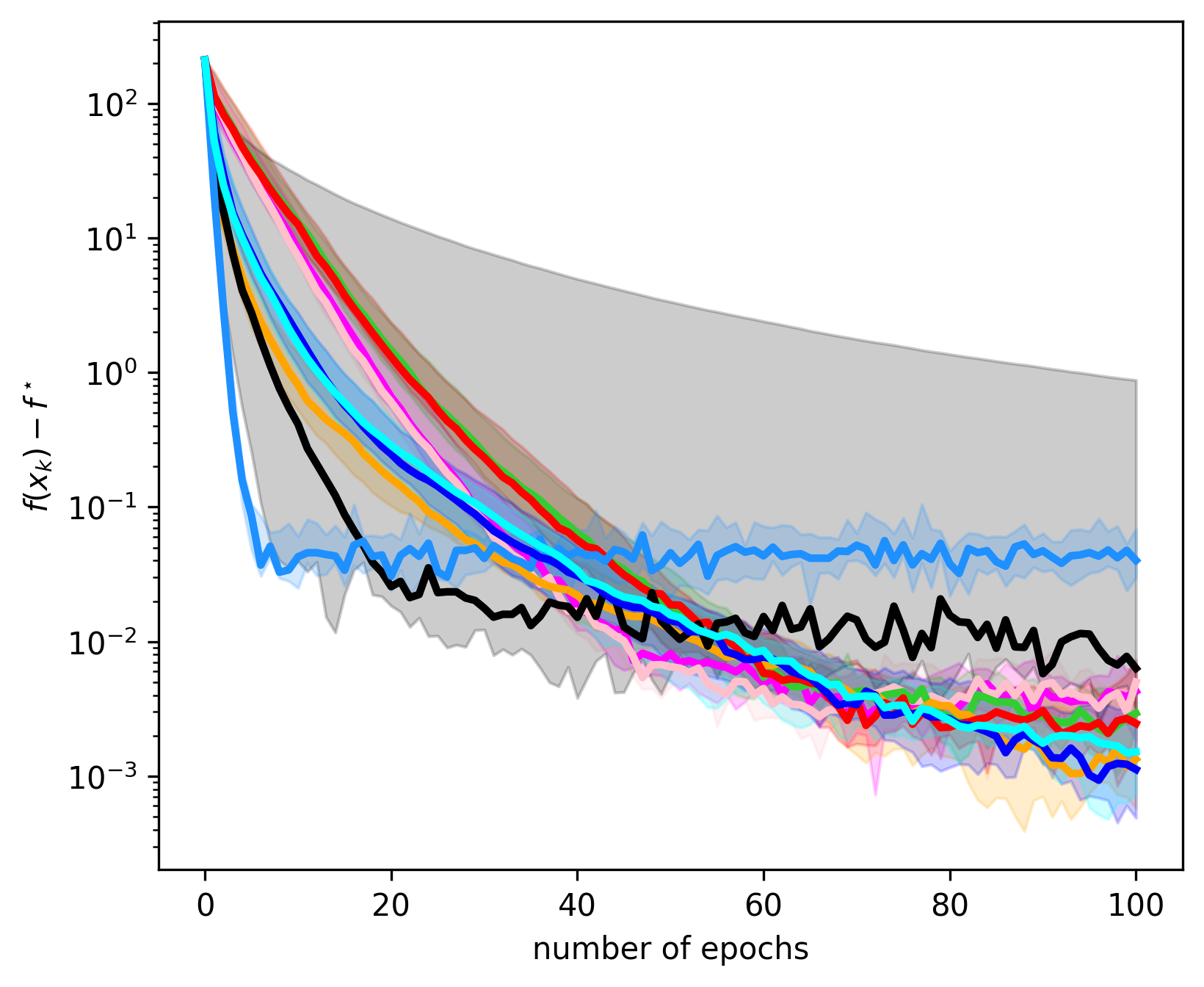} 
			}
			\\
			\makecell{
				\footnotesize Logistic -- 2moons\\
				\includegraphics[width=0.25\textwidth]{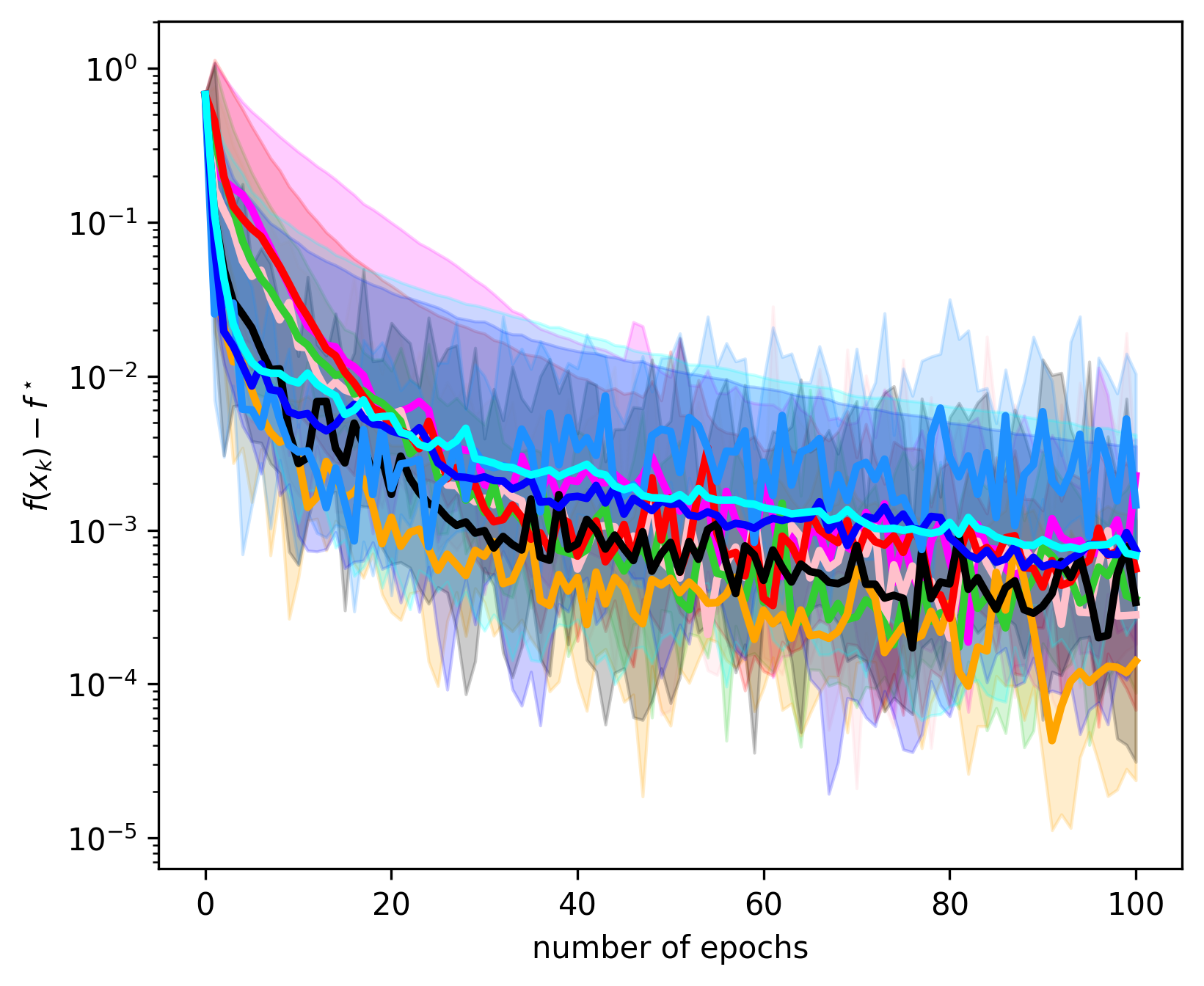}
			}&
			\makecell{
				\footnotesize Logistic -- w8a\\
				\includegraphics[width=0.25\textwidth]{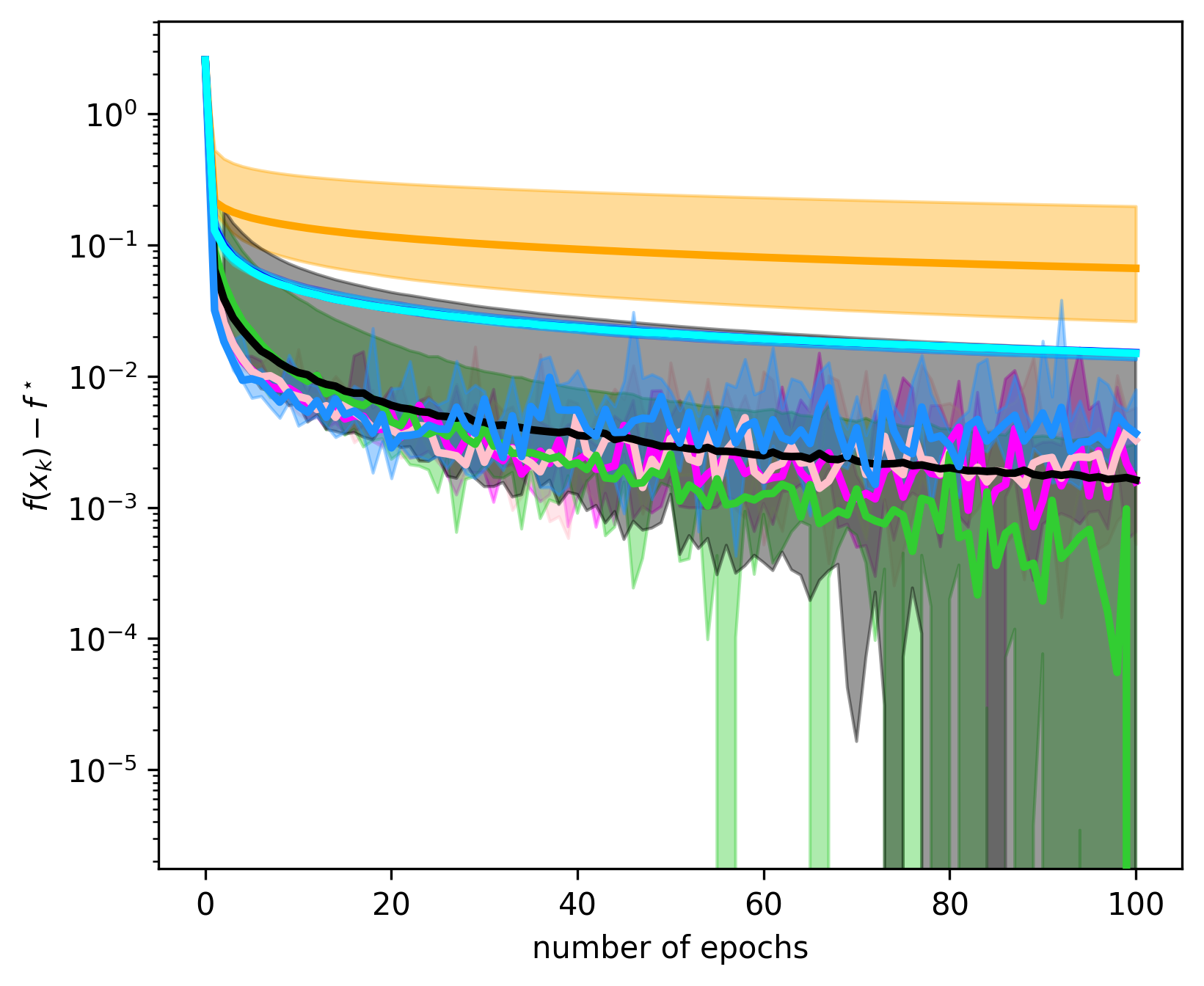}
			}&
			\makecell{
				\footnotesize Poisson -- Synthetic\\
				\includegraphics[width=0.25\textwidth]{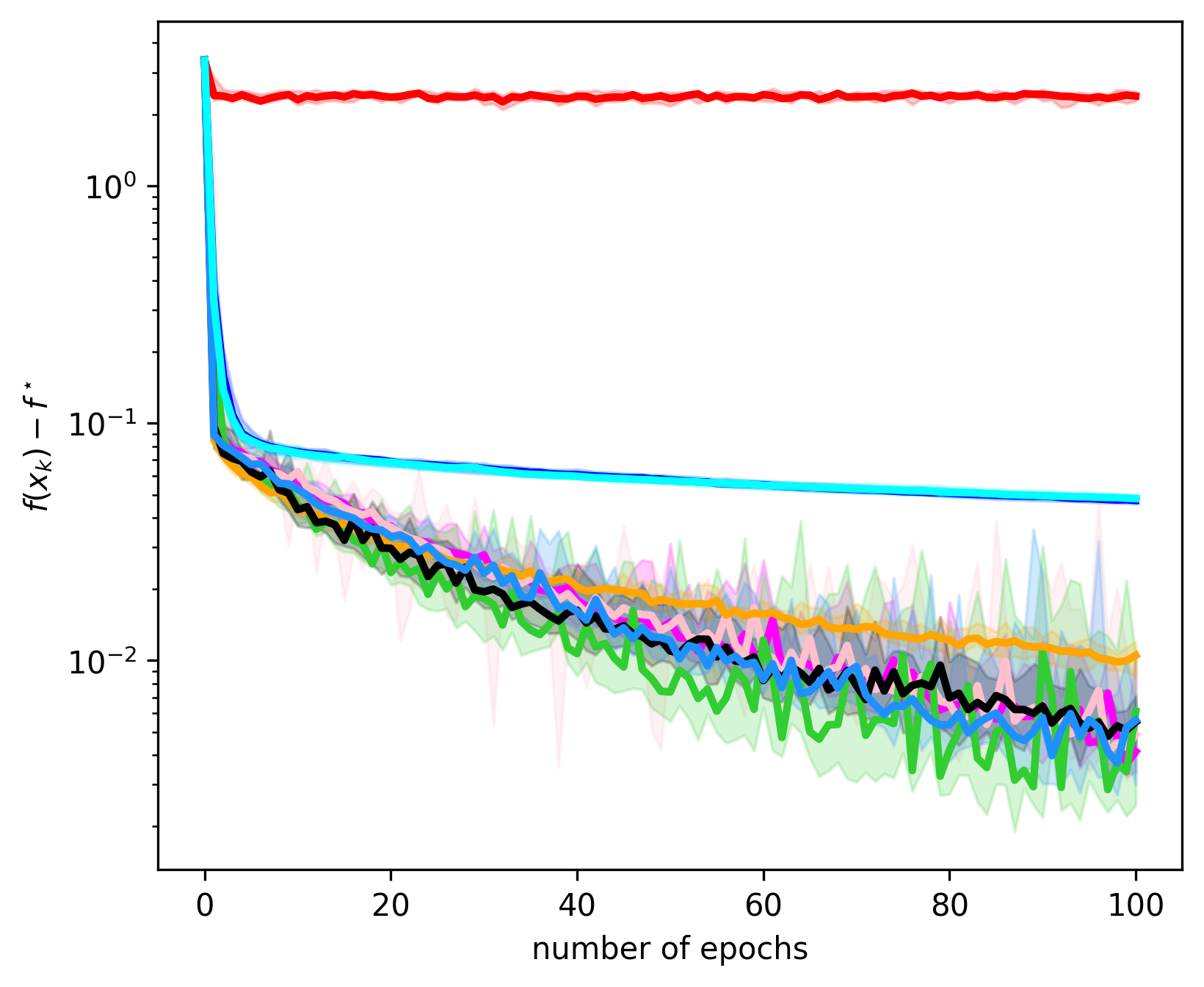}
			}
		\end{tabular}
		\caption{Evolution of the sub-optimality  $f(x_k)-f^\star$ over the epochs. Comparison between the three variants of AdaSGD used \emph{without tuning} and other methods that are tuned via a dense grid-search. Solid curves: median value across ten runs, area: gap between $10\%$ and $90\%$ quantiles.} \label{fig::EXP_loss}
	\end{figure*}

	To assess the performance of AdaSGD, we consider six different stochastic convex differentiable optimization problems of the form \eqref{eq:optim}, with synthetic and real data. We consider four types of loss functions that illustrate the possible different settings from Section~\ref{sec:conv}: linear and ridge regression that fit \ref{case::linreg} and \ref{case::ridge}, respectively; logistic regression that does not fit the cases considered, and Poisson regression where $\gf{}$ is locally Lipschitz continuous but not globally.
	We compare our method to SGD with fixes step-size $\lambda_{0}$ (no decay), SGD with decay: $\lambda_{k} = \frac{\lambda_{0}}{k^{1/2+\delta}}$ ($\delta=10^{-2}$) and the two stochastic versions from \citet{malitsky2019adaptive} (AdaSGD--MM) that also feature a tunable step-size $\lambda_0$ (previously denoted by $\alpha$ in Section~\ref{sec:prev}).
	We also consider two additional methods with adaptive step-sizes: the step-size version of Adagrad and the stochastic Polyak step-size \citep{loizou2021stochastic}. Note that these methods still feature a ``step-size'' like parameter (denoted by $\lambda_0$ in the figures for consistency).
	We provide further details on the experiments in Section~\ref{app:expdetails} of the Appendix and a publicly available code.\footnote{\url{https://github.com/camcastera/AdaSGD}}

	\noindent
	\textbf{Sensitivity to the initial step-size.}
	We first study how the choice of the initial step-size $\lambda_0$ affects each of the seven algorithms. The results for the first three problems on Figure~\ref{fig::stepsize_sensitivity} and the three others on Figure~\ref{fig::detailed_sensitivity}  show that AdaSGD is indeed much less sensitive to $\lambda_0$ than other algorithms (except for a few outliers in logistic regression, but the median remains stable even there). This significantly reduced sensitivity also does not seem to influence the performance as the value of $f$ achieved is comparable or better than that of optimally-tuned SGD. We also remark that despite being adaptive, the step-size Adagrad and Stochastic Polyak are more sensitivities to $\lambda_0$ than our method.
	We note nonetheless that V-\textbf{I} seems to have a significantly different behavior than V-\textbf{II} and V-\textbf{III} (which are relatively close). While V-\textbf{I} comes with less strong theoretical guarantees, it is sometimes better (or worse) than the other two variants depending on the problem. This is not surprising as vanilla SGD also peforms sometimes better with or without decay depending on the problem.
	Figures~\ref{fig::stepsize_sensitivity} and~\ref{fig::detailed_sensitivity} clearly show that our method fixes the issue from the stochastic method of \citet{malitsky2019adaptive} as for the latter, $\lambda_0$ has a crucial impact on performance. This is the main goal and contribution of this work.
	An important take-away for practitioners is that we see that \emph{taking low values of $\lambda_0$ for AdaSGD seem to always be a good choice}, therefore we can use any variant of AdaSGD with (for example) $\lambda_0=10^{-3}$ and avoid tuning the step-size.

	\noindent
	\textbf{Minimization performance.}
	We ran AdaSGD with the default value (\emph{without tuning} $\lambda_0$) and compare it to the other algorithms that were tuned via a dense grid search on six problems. Results are shown on Figure~\ref{fig::EXP_loss}.
	We note that V-\textbf{I} sometimes struggles, which illustrates the importance of the decay introduced in V-\textbf{II} and V-\textbf{III}. They exhibit fewer oscillations due to the control of the noise (the additional term in \eqref{eq::Lyap}). Interestingly, V-\textbf{I} works better than the other two on Poisson regression, which is a problem where there is no global Lipschitz continuity of the gradients.
	We observe that the performance of AdaSGD is comparable to that of the other methods.
	The gain or loss is often marginal but AdaSGD has the significant benefit that no tuning is required to achieve this performance.
	AdaSGD requires two gradient evaluations per iteration but this is compensated by the absence of grid search, which itself requires many gradient oracle calls.

	\noindent
	\textbf{Effect of the batch-size.} To investigate the effect that the size of the mini-batches have on the stability to $\lambda_0$, we repeated some of the experiments in the case of high noise (small mini-batches) and low noise (large mini-batches). The results presented in Figure~\ref{fig::smallerMB} and Figure~\ref{fig::largerMB} in the appendix show that our conclusions remain the same: a small choice of $\lambda_0$ is always good for our three methods and they exhibit minimization performances comparable to those of most optimally-tuned methods considered.

	\section{Conclusion}
	We introduced AdaSGD, an adaptation of AdaGD to the stochastic setting that coincides with AdaGD in the full-batch case. 
	AdaSGD transfers the advantages of AdaGD to the stochastic setting, some of which were lost in the original stochastic adaptation of AdaGD. Our method does not rely on any heuristics and is rather derived via a careful Lyapunov analysis, resulting in three variants step-size strategies.
	We conducted a thorough analysis, showing convergence and rates for our methods. These guarantees however rely on stronger assumptions than in the deterministic setting, which is due to the noise induced by stochastic gradients. 
	We show promising experiments that evidence that our primary goal is achieved: AdaSGD is very robust to the choice of the initial step-size without sacrificing minimization performance.
	
	The main directions for further research revolve around relaxing the requirements related to convexity. 
	In Appendix~\ref{app:convex}, we provide partial theoretical results that suggest that the convergence results of Section~\ref{sec:conv} could possibly be extended to all convex functions. On the other-hand, adapting AdaSGD to non-convex functions would allow training neural networks but represents a significant challenge since the derivation of Section~\ref{sec:lyap} relies heavily on the convexity inequality.

	
	\FloatBarrier 
	
	
	\appendix
	
	\section{Ensuring Square-summability of the Step-sizes} \label{sec:model_App}

	Theorem~\ref{thm:convergence} holds true when the sequence of step-sizes $(\lambda_k)_{k\in\N}$ is square-summable. 
	We can guarantee this in \ref{case::strcvx}, \ref{case::linreg} and \ref{case::ridge} as discussed in the main text. We now prove this and provide further discussion on these three settings.

	\subsection{\ref{case::strcvx}: Strong Convexity}\label{app:strcvx}
	
	In \ref{case::strcvx}, we assume that all the functions $f_\xi$ are strongly convex, which is a sub-class of convex functions. We recall that this means that there exists $\mu>0$, such that for all $\xi$, and $x,y\in\R^d$,
	$$
	f_{\xi}(y) \geq  f_{\xi}(x) +   \langle \gf[\xi](x) , y - x \rangle +  \frac{\mu}{2} \| x - y\|^2,
	$$
	or equivalently
	\begin{equation}\label{eq::gradStrMonotone}
		\langle  \gf[\xi](x) - \gf[\xi](y) , x - y \rangle \geq \mu \| x - y\|^2.
	\end{equation}
	Note that we assume that all the functions $f_\xi$ have the same strong convexity parameter $\mu$. We could have equivalently assumed that there exists $\mu>0$ such that each $f_\xi$ is $\mu_\xi$ strongly convex with $\mu_\xi\geq \mu>0$. In the case where $f$ is a finite-sum of functions this is always possible (by taking $\mu=\min_\xi \mu_\xi$).
	
	Strong convexity allows upper-bounding the step-sizes $(\lambda_k)_{k\in\N}$. Indeed, by using Cauchy-Schwarz inequality on \eqref{eq::gradStrMonotone}, we get,
	$$
	\norm{\gf[\xi](x) - \gf[\xi](y)} \norm{x - y} \geq \mu \| x - y\|^2.
	$$
	So in particular, using this on the iterates $(x_k)_{k\in\N}$ and mini-batches $(\xi_k)_{k\in\N}$ in Algorithm~\ref{ass:VarianceControl}, for all $k\geq 1$
	\begin{equation}
		\frac{\norm{x_k-x_{k-1}}}{ \norm{\gf[\xi_{k-1}](x_k) - \gf[\xi_{k-1}](x_{k-1})}} \leq \frac{1}{\mu}. \label{eq:boundMu}
	\end{equation}
	Hence, from the definition of the step-size $\lambda_k$ in \eqref{eq:choicestep}, we deduce that,
	\begin{equation*} 
		\lambda_k  \leq \begin{cases}
			\frac{1}{2\sqrt{2}\mu},
			\quad &\text{(V-\textbf{I})}
			\\
			\frac{1}{k^{1/2+\delta}}\frac{1}{2\sqrt{2}\mu}
			\quad &\text{(V-\textbf{II} and V-\textbf{III})}
		\end{cases} 
	\end{equation*}
	Therefore, for Variants \textbf{II} and \textbf{III} of the step-size, we can ensure that,
	$$
	\sum_{k=0}^{+\infty} \lambda_k^2 \leq \sum_{k=0}^{+\infty} \frac{1}{k^{1+2\delta}} < +\infty.
	$$
	Therefore, Theorem~\ref{thm:convergence} holds true in \ref{case::strcvx}. Why strong convexity may seem to be the natural assumption for ensuring square-summability, a similar result can be derived in non-strongly convex settings.
	

	\subsection{\ref{case::linreg}: Linear Regression.} 
	We now consider the classical machine learning problem of linear regression, where $f$ is a finite sum \eqref{eq:optim_fullbatch} and the $f_{\ell}$'s need not be strongly convex. We assume that for all $x\in\R^d$, the functions $(f_\ell)_{\ell\in\{1,\ldots,N\}}$ take the following form:
	$$
	f_\ell(x) = \frac{1}{2}(y_\ell - \inner{w_\ell}{x})^2,
	$$
	where $w_\ell$ is a $d$-dimensional vector of predictors (with $w_\ell \neq 0$) and $y_\ell$ is a scalar response variable. The discussion can easily be extended to the case where $w_\ell$ are matrices and $y_\ell$ are vectors, which we do not consider for the sake of simplicity.
	
	For all $\ell\in\{1,\ldots,N\}$, and $x\in\R^d$, we have
	\begin{equation}
		\gf[\ell](x) = (\inner{w_\ell}{x} - y_\ell) w_\ell, \label{eq:grad_lin}
	\end{equation}
	and the Hessian of $f_{\ell}$ is given by
	$$
	\nabla^2 f_{\ell}(x) =  w_\ell  w_\ell^{T}, \quad \mbox{ for all } x \in \R^d.
	$$
	Since $w_\ell  w_\ell^{T}$ is a rank-1 matrix, the Hessian admits a zero eigenvalue (actually it even has $d-1$) so the function $f_\ell$ is not strongly convex.

	We now use this on the iterates of Algorithm~\ref{algo::AdaSGD}. For the sake of simplicity, we assume, without loss of generality, that the mini-batches are single valued, \textit{i.e.} for all $k \in \NN$, $\xi_k\in\{1,\ldots,N\}$. 
	Then remark that for any $k \in \NN$ of Algorithm~\ref{algo::AdaSGD}, \eqref{eq:grad_lin} implies that
	\begin{align}\label{eq:decomplinreg}
		\begin{split}
			\norm{\gf[\xi_{k-1}](x_k) - \gf[\xi_{k-1}](x_{k-1})}^2  
			&=  \norm{ (y_{\xi_{k-1}} - \inner{w_{\xi_{k-1}}}{x_{k}} ) w_{\xi_{k-1}} - (y_{\xi_{k-1}} - \inner{w_{\xi_{k-1}}}{x_{k-1}} ) w_{\xi_{k-1}}}^2 \nonumber \\
			& =  \norm{   \inner{w_{\xi_{k-1}}}{x_{k} - x_{k-1}}  w_{\xi_{k-1}}}^2 =   \inner{w_{\xi_{k-1}}}{x_{k} - x_{k-1}}^2 \norm{w_{\xi_{k-1}}}^2. 
		\end{split}
	\end{align}
	But remark that from the definition of Algorithm~\ref{algo::AdaSGD},
	$$
	x_{k} - x_{k-1} = -\lambda_{k-1}w_{\xi_{k-1}} =  - \lambda_{k-1} \alpha_{k-1} w_{\xi_{k-1}},
	$$ 
	where $\alpha_{k-1} = (\inner{w_{\xi_{k-1}}}{x} - y_{\xi_{k-1}})$.
	So inserting the above in \eqref{eq:decomplinreg},
	\begin{equation}\label{eq:decomplinreg}
		\norm{\gf[\xi_{k-1}](x_k) - \gf[\xi_{k-1}](x_{k-1})}^2  
		=  \lambda_{k-1}^2\alpha_{k-1}^2\norm{w_{\xi_{k-1}}}^6. 
	\end{equation}
	Then similarly,
	$$
	\norm{x_{k} - x_{k-1}}^2 = \lambda_{k-1}^2 \alpha_{k-1}^2 \|w_{\xi_{k-1}}\|^2,
	$$
	So we obtain that
	$$
	\frac{\norm{x_k-x_{k-1}}}{ \norm{\gf[\xi_{k-1}](x_k) - \gf[\xi_{k-1}](x_{k-1})}} = \frac{1}{ \|w_{\xi_{k-1}}\|^2}.
	$$
	Finally, by using the finite-sum structure, we denote by $\hat{\mu} = \max_{\ell \in\{1,\ldots,N\}}{\|w_{\ell}\|^2}
	$ and we arrive to the same inequality as in \eqref{eq:boundMu}, \emph{without} relying on strong convexity. Therefore the same reasoning as in \ref{case::strcvx} applies here and we deduce that for 
	Variants \textbf{II} and \textbf{III} of the step-size,
	$
	\sum_{k=0}^{+\infty} \lambda_k^2 < +\infty.
	$

	\subsection{\ref{case::ridge}: Sum of Ridge Functions}
	As a third setting, we consider a generalization of linear regression  where the function $f$ is a finite sum of convex ridge functions. 
	More precisely, we assume that, for each $\ell\in\{1,\ldots,N\}$, the function $f_\ell$ takes the form
	\begin{equation*}
		f_\ell(x) = g_\ell\left(\inner{w_\ell}{x}\right), 
	\end{equation*}
	where $w_\ell$ is defined as in \ref{case::linreg}, and $g_\ell\colon \R\to \R$ is a known function that is differentiable and $\mu_\ell$-strongly convex on $\R$. This setting includes the previous example of linear regression with $g_\ell(u) = \frac{1}{2} (u-y_\ell)^2$, for all $u\in\R$.
	
	Linear combinations of ridge functions constitute a functional class that is a central tool in the approximation theory of multivariate functions, and for the study of neural networks with a single or two hidden layers.  For a survey on these topics  we refer to \citep{MR4390796}. Using sums of ridge functions has also proved to be relevant for the regularization of inverse problems \citep{MR4635236}, the study of convex optimization in the space of measures \citep{MR4529991}, and for adversarial bandit problems \citep{10.1214/24-AOS2395} to name a few.

	Note that assuming that as for linear regression, the $f_\ell$ are not strongly convex (they are only in one direction given by $w_\ell$), and that $f$ needs not to be strongly convex as well. Hence, this setting remains different from \ref{case::strcvx}.

	Similarly to \ref{case::linreg}, for all $\ell\in\{1,\ldots,N\}$ and $x\in\R^d$,
	\begin{equation}
		\gf[\ell](x) = g_{\ell}'(\inner{w_\ell}{x}) w_\ell. \label{eq:grad_ridge}
	\end{equation}
	
	Again, assuming without loss of generality, that $k \in \NN$, $\xi_k\in\{1,\ldots,N\}$, using the strong convexity of $g_{\xi}$, the above implies for Algorithm~\ref{algo::AdaSGD} that, for all $k\geq 1$, 
	\begin{multline*}
		\norm{\gf[\xi_{k-1}](\xk) - \gf[\xi_{k-1}](\xkm)}^2 
		=  \left(g'_{\xi_{k-1}}\left( \inner{w_{\xi_{k-1}}}{\xk}\right) - g'_{\xi_{k-1}}\left( \inner{w_{\xi_{k-1}}}{\xkm}\right)\right)^2 \norm{w_{\xi_{k-1}}}^2 \\
		\geq   \mu_{\xi_{k-1}}^2 \left(\inner{w_{\xi_{k-1}}}{\xk-\xkm}
		\right)^2 \norm{w_{\xi_{k-1}}}^2 =    \mu_{\xi_{k-1}}^2 \left(\inner{w_{\xi}}{ \lambda_{k-1}\gf[\xikm](\xkm)}\right)^2 \norm{w_{\xi}}^2.
	\end{multline*}
	Moreover from \eqref{eq:grad_ridge}, one has that
	$$ \inner{w_{\xi}}{\gf[\xikm](\xkm)} = g'_{\xi}\left( w_{\xi}^T x\right) \norm{w_{\xi}}^2,$$
	consequently
	$$
	\norm{\gf[\xikm](\xk) - \gf[\xikm](\xkm)}^2 \geq  \mu_\xikm^2\lambda_{k-1}^2 g'_{\xikm}\left( w_{\xikm}^T x\right) \norm{w_{\xikm}}^6,
	$$
	and similarly to \ref{case::strcvx}, we obtain the following upper bound
	\begin{eqnarray*}
		\frac{\norm{\xk-\xkm}}{\norm{\gf[\xikm](\xk) - \gf[\xikm](\xkm)}} 
		& \leq &
		\frac{ 1 }{\mu_\xikm \norm{w_{\xikm}}^2}.
	\end{eqnarray*}
	From that point we repeat the same argument as in the previous section to deduce that 
	$
	\sum_{k=0}^{+\infty} \lambda_k^2 < +\infty.
	$

	\section{Missing Proofs and Convergence Results}
	
	We start by showing the lemmas used to control the ``noisy'' term in the beginning of Section~\ref{sec:conv}.
	\subsection{Variance Transfer}\label{sec::varTransfer}
	
	We first recall the definition of Lipschitz continuity of the gradient.
	\begin{definition}\label{def::Lsmooth}
		A function $g\colon \R^d\to\R$ is said to be locally smooth or to have locally Lipschitz continuous gradient, if $g$ is differentiable and for all $x\in\R^d$, there exists $L>0$ and a neighborhood $V$ of $x$ such that, for all $x,y\in V$,
		\begin{equation*}
			\norm{\gf(x)-\gf(y)}\leq L \norm{x-y}.
		\end{equation*}
		If $V=\R^d$, the function is simply said to be (globally) $L$-smooth or to have $L$-Lipschitz continuous gradient.
	\end{definition}

	\begin{lemma} \label{lemma_bound}
		For all $x\in\R^d$, and random variables ${\xi_1},{\xi_2}$ as defined in Section~\ref{sec::intro},
		\begin{multline*}
			\esp{\norm{\gf[\xi_1](x) - \gf[\xi_{2}](x)}^2} 
			\leq 4 \esp{\norm{\gf[\xi_1](x) - \gf[\xi_1](x^\star)}^2} 
			+ 4 \esp{\norm{\gf[\xi_{2}](x) - \gf[\xi_{2}](x^\star)}^2}
			\\
			+ 4 \esp{\norm{\gf[\xi_1](x^\star)}^2} +  4\esp{\norm{\gf[\xi_{2}](x^\star)}^2}.
		\end{multline*}
	\end{lemma}
	
	\begin{proof}
		For all $x\in\R^d$, we have
		\begin{multline*}
			\norm{\gf[\xi_1](x) - \gf[\xi_{2}](x)}^2 \leq 2 \norm{\gf[\xi_1](x)}^2  + 2 \norm{\gf[\xi_{2}](x)}^2
			\\  =  2 \norm{\gf[\xi_1](x) - \gf[\xi_1](x^\star) + \gf[\xi_1](x^\star)}^2  + 2\norm{\gf[\xi_{2}](x) - \gf[\xi_2](x^\star) + \gf[\xi_2](x^\star)}^2
			\\  \leq 4 \norm{\gf[\xi_1](x) - \gf[\xi_1](x^\star)}^2 + 4 \norm{\gf[\xi_1](x^\star)}^2  + 4\norm{\gf[\xi_{2}](x) - \gf[\xi_2](x^\star)}^2 + 4\norm{\gf[\xi_2](x^\star)}^2.
		\end{multline*}
		The result is then obtained by taking expectations on both sides of the above.
	\end{proof}

	We can now prove Lemma~\ref{lem::uniformControl}.
	
	\begin{proof}[Proof of Lemma~\ref{lem::uniformControl}]
		First remark that if Assumption~\ref{ass:VCi} holds, then 
		there is a finite number of functions $f_\xi$, according to \eqref{eq:optim_fullbatch}. Thus, there exists a global Lipschitz constant $L=\max_{\ell \in \{1,\ldots,N\}} L_\ell$ on all the $\gf[\xi]$'s. If we rather make Assumption~\ref{ass:VCiii}, then the iterates lie in a compact subset $\mathsf{K}$ of $\R^d$ and each $f_\xi$is globally\footnote{This is because $\gf[\xi]$ is locally Lipschitz continuous and $\mathsf{K}$ can be covered by a finite union of open sets.} smooth on $\mathsf{K}$. Then by using the same argument as above, there exists a uniform Lipschitz constant $L$ on all the $\gf[\xi]$'s. Overall, we have shown in particular that under any of the three set of assumptions, there exists $L$, such that $\forall k\geq 1$, 
		$$\norm{\gf[\xi](x_k) - \gf[\xi](x^\star)} \leq L \norm{x_k - x^\star}.$$

		For all $k\geq 1$, let us consider $(x_k)_{k\in\N}$, $(\xi_k)_{k\in\N}$ and $(\lambda_k)_{k\in\N}$ as defined in Algorithm~\ref{algo::AdaSGD}. Then by direct application of Lemma~\ref{lemma_bound}, we have
		\begin{multline*}
			\esp[k-1]{\norm{\gf[\xi_k](x_k) - \gf[\xi_{k-1}](x_k)}^2} 
			\\
			\leq 4 \esp[k-1]{\norm{\gf[\xi_{k}](x_k) - \gf[\xi_{k}](x^\star)}^2} + 4 \esp[k-1]{\norm{\gf[\xi_{k-1}](x_k) - \gf[\xi_{k-1}](x^\star)}^2} \\
			+ 4 \esp[k-1]{\norm{\gf[\xi_{k}](x^\star)}^2} +  4\esp[k-1]{\norm{\gf[\xi_{k-1}](x^\star)}^2}.
		\end{multline*}
		
		By definition, $\sigma_k^2 = \esp[k-1]{\norm{\gf[\xi_{k}](x^\star)}^2} +  \esp[k-1]{\norm{\gf[\xi_{k-1}](x^\star)}^2}$ and using the uniform Lipschitz-continuity discussed above, we obtain,
		\begin{equation}\label{eq::Esp_bound_on_noise}
			\esp[k-1]{\norm{\gf[\xi_k](x_k) - \gf[\xi_{k-1}](x_k)}^2} 
			\leq 8L^2 \esp[k-1]{\norm{x_k - x^\star}^2}  = 8L^2 \norm{x_k - x^\star}^2
			+ 4\sigma_k^2,
		\end{equation}
		where the last equality is obtained by noticing that $x_k$ is not random with respect to the filtration $\mathcal{F}_{k-1}$.
		
		It only remains to prove that $\forall k\geq 1$, there exists $\sigma>0$ such that $\sigma_k^2\leq \sigma^2$. This is directly assumed in Assumption~\ref{ass:VCii}. In the other-two cases, $f$ has the finite-sum structure \eqref{eq:optim_fullbatch}. Hence, remark that in that case since $\norm{\gf[\xi_{k}](x^\star)}^2$ does not depend on $x_k$, we can simply take the uniform bound: 
		$$
		\sigma^2 = \max_{\ell,j\in\{1,\ldots,N\}} \norm{\gf[\ell](x^\star)}^2 + \norm{\gf[j](x^\star)}^2.
		$$	
		Using this in \eqref{eq::Esp_bound_on_noise}, we get the desired result.
	\end{proof}


	\subsection{Convergence Analysis\label{sec:as_App}} \label{sec:conv_App}

	All the convergence results in Section~\ref{sec:conv} rely on the Robbins-Siegmund Theorem, a standard tool for proving convergence of random sequences. We refer the reader to \citep{Duflo1997} for further discussion on this result.
	
	\begin{theorem}[Robbins-Siegmund] \label{th:RS}
		Let $(U_k)$, $(V_k)$, $(\alpha_k)$, and $(\beta_k)$ be sequences of nonnegative and integrable random variables on some arbitrary probability space that are adapted to a filtration $\mathcal{F}_k$,  and such that, almost surely,
		\begin{itemize} \item[(i)]
			$
			\sum_{k \geq 0} \alpha_k < + \infty \quad \mbox{ and } \quad \sum_{k \geq 0} \beta_k < + \infty,
			$
			\item[(ii)]	for all $ k \geq 1$,
			$$
			\esp{V_{k} | \mathcal{F}_{k-1}}  \leq V_{k-1}(1+ \alpha_{k}) - U_{k} + \beta_{k}.
			$$
		\end{itemize}
		Then, almost surely, $ V_k$ converges to a random variable $V_{\infty}$, and 
		$$
		\sum_{k \geq 0} U_k < + \infty \ \mbox{ almost surely.}
		$$
	\end{theorem}

	\paragraph{Proof of almost sure convergence}
	We will now use this theorem to prove Theorem~\ref{thm:convergence}.

	\begin{proof}[Proof of Theorem~\ref{thm:convergence}]
		
		We use the same notations as in Section~\ref{sec:lyap}. According to  Proposition \ref{prop:lyap}, for Algorithm~\ref{algo::AdaSGD}, for all $k\geq 2$, it holds that 
		\begin{equation*} 
			\esp[k-1]{T_{k}}  \leq T_{k-1} + 4  \lambda_{k}^2\esp[k-1]{\norm{\gf[\xi_{k}](x_{k}) - \gf[\xi_{k-1}](x_{k})}^2},
		\end{equation*}
		where $T_k$ is defined in \eqref{eq::LyapTk}.
		Then by using Assumption~\ref{ass:VarianceControl} we can apply Lemma~\ref{lem::uniformControl}, so there exists $L,\sigma>0$ such that for all $k\in\N$,
		\begin{equation*} 
			\esp[k-1]{T_{k}}  \leq T_{k-1} + 32 \lambda_k^2 L^2 \norm{x_k - x^\star}^2 + 16\lambda_k^2\sigma^2
		\end{equation*}
		Notice from the definition of $T_k$ in \eqref{eq::LyapTk}, that $\norm{x_k - x^\star}^2\leq T_{k-1}$, so
		\begin{equation} \label{eq::beforeRS}
			\esp[k-1]{T_{k}}  \leq 
			(1+32 L^2 \lambda_{k}^2)
			T_{k-1} + 16\lambda_k^2\sigma^2
		\end{equation}
		
		We now apply Theorem~\ref{th:RS} to \eqref{eq::beforeRS} with 
		$$
		V_{k} = T_{k}, \; \alpha_{k} = 32 L^2 \lambda_{k}^2,  \;  \beta_{k} = 16\lambda_k^2\sigma^2, \text{ and }  
		U_{k} = 0.
		$$
		
		Therefore, as stated in Theorem~\ref{thm:convergence}, whenever $\sum_{k \geq 0} \lambda_{k}^2 < + \infty$, then $\sum_{k \geq 0} \alpha_k < + \infty$ and $\sum_{k \geq 0} \beta_k < + \infty$ so we can apply  Theorem~\ref{th:RS} to deduce that $T_k$ converges almost surely.
	\end{proof}

	\paragraph{Proof of convergence rates}
	
	
	We will now prove Theorem~\ref{thm::CVrates}. We keep again the notations of Section~\ref{sec:lyap} and consider Algorithm~\ref{algo::AdaSGD} with V-III of step-size \eqref{eq:choicestep}. This Theorem holds only for \ref{case::strcvx}, which is the one where each $f_\xi$ is $\mu$-strongly convex. Importantly, this implies that $f$ is $\mu$-strongly convex as well.
	We begin with two technical lemmas that will be useful for proving Theorem~\ref{thm::CVrates}.
	
	{
		\begin{lemma}\label{lem:boundLambdaVariant3}
			Let $(\lambda_k)_{k\in\N}$ denote the step-sizes of Algorithm~\ref{algo::AdaSGD} with V-\textbf{III} of \eqref{eq:choicestep}.
			Then under Assumption~\ref{ass:VarianceControl} and \ref{case::strcvx}, for all $k\geq 2$,
			\begin{equation}
				\frac{1}{k^{1/2+\delta}}\frac{1}{2\sqrt{2}L} \leq  \lambda_k \leq \frac{1}{k^{1/2+\delta}}\frac{1}{2\sqrt{2}\mu}. \label{eq:boundLmu}
			\end{equation}
		\end{lemma}
		\begin{proof}
			The upper-bound in \eqref{eq:boundLmu} was already proved in Appendix~\ref{app:strcvx}. As for the lower-bound, we proceed similarly. 
			By applying the same arguments as in the proof of Lemma~\ref{lem::uniformControl}, we deduce that there exists $L>0$ such that each $f_\xi$ has $L$-Lipschitz continuous gradient, at least on a compact set containing the iterates $(x_k)_{k\in\N}$. This means that
			$$
			\norm{\gf[\xi_{k-1}](x_k) - \gf[\xi_{k-1}](x_{k-1})} \leq L \norm{x_k-x_{k-1}},
			$$ and thus,
			$$
			\frac{1}{k^{1/2+\delta}}\frac{1}{2\sqrt{2}L}
			\leq
			\frac{1}{k^{1/2+\delta}}\frac{\norm{x_k-x_{k-1}}}{2\sqrt{2}\norm{\gf[\xi_{k-1}](x_k) - \gf[\xi_{k-1}](x_{k-1})}}.
			$$
			Unlike the upper-bound in \eqref{eq:boundLmu}, this is not enough to conclude. Yet, for all $k\geq 2$,
			\begin{align*}
				\lambda_k & = \min \left\{ \frac{1}{k^{1/2+\delta}}\frac{\norm{x_k-x_{k-1}}}{2\sqrt{2}\norm{\gf[\xi_{k-1}](x_k) - \gf[\xi_{k-1}](x_{k-1})}},   \lambda_{k-1} \sqrt{ 1 +     \left(1-\frac{1}{k^{1/2+\delta}}\right)\frac{ \lambda_{k-1} }{\lambda_{k-2}}}   \right\}
				\\
				& \geq 
				\min \left\{ \frac{1}{k^{1/2+\delta}}\frac{1}{2\sqrt{2}L},   \lambda_{k-1} \sqrt{ 1 +     \left(1-\frac{1}{k^{1/2+\delta}}\right)\frac{ \lambda_{k-1} }{\lambda_{k-2}}}   \right\}
				\\&\geq 
				\min \left\{ \frac{1}{k^{1/2+\delta}}\frac{1}{2\sqrt{2}L},   \lambda_{k-1}   \right\},
			\end{align*}
			where in the last line where used the fact that for $k\geq 2$, $\lambda_{k-1} \sqrt{ 1 +     \left(1-\frac{1}{k^{1/2+\delta}}\right)\frac{ \lambda_{k-1} }{\lambda_{k-2}}}\geq \lambda_{k-1}$.
			
			Therefore, by straightforward induction,
			\begin{align*}
				\lambda_k & \geq 
				\min \left\{ \frac{1}{k^{1/2+\delta}}\frac{1}{2\sqrt{2}L},   \lambda_{k-1}   \right\} 
				\geq 
				\min \left\{ \frac{1}{k^{1/2+\delta}}\frac{1}{2\sqrt{2}L}, 
				\frac{1}{k^{1/2+\delta}}\frac{1}{2\sqrt{2}L}
				, \lambda_{k-2}   \right\} 
				\\
				& \geq \ldots 
				\geq 
				\min \left\{ \frac{1}{k^{1/2+\delta}}\frac{1}{2\sqrt{2}L}
				, \lambda_{1}   \right\}.
			\end{align*}
			Since $\lambda_1 = \frac{\norm{x_1-x_0}}{2\sqrt{2}\norm{\gf[\xi_{0}](x_1) - \gf[\xi_{0}](x_{0})}} \geq \frac{1}{2\sqrt{2}L}$, we obtain the lower-bound in \eqref{eq:boundLmu}.

		\end{proof}
	}
	
	Before presenting the next technical lemma, we recall the notation for all $k\geq 1$, $\theta_k = \lambda_k/\lambda_{k-1}$.
	
	\begin{lemma}\label{lem::BoundsThetak}
		Let $(\lambda_k)_{k\in\N}$ denotes the step-sizes of Algorithm~\ref{algo::AdaSGD} with V-\textbf{III} of \eqref{eq:choicestep}.
		Then under Assumption~\ref{ass:VarianceControl} and \ref{case::strcvx}, for all $k\geq 2$,
		\begin{align}
			\label{eq::boundTheta}
			\frac{1}{2^{\delta}\sqrt{2} }\frac{\mu}{L} & \leq \theta_k
			\\
			\frac{1}{k^{1/2+\delta}}\frac{\mu}{2^\delta4L^2} 
			&\leq \lambda_k\theta_k  
			\leq \lambda_{k-1}\left(1 + \theta_{k-1}\right) \left(1  - \frac{1}{k^{1/2+\delta}} \frac{\mu}{\mu+2^\delta\sqrt{2}L}\right),
			\label{eq::boundLambdaTheta}
		\end{align}
		where $\theta_k = \frac{\lambda_k}{\lambda_{k-1}}.$
	\end{lemma}
	
	\begin{proof}
		As before we can make use of $\mu$-strong convexity and $L$-Lipschitz continuity of the gradient.
		In particular, using Lemma~\ref{lem:boundLambdaVariant3}, we deduce that for all $k\geq 2$,
		\begin{equation*}
			\theta_k = \frac{\lambda_k}{\lambda_{k-1}}  \geq
			\frac{(k-1)^{1/2+\delta}}{k^{1/2+\delta}} \frac{2\sqrt{2}\mu}{2\sqrt{2}L} = \left(\frac{k-1}{k}\right)^{1/2+\delta} \frac{\mu}{L}.
		\end{equation*}
		We then get \eqref{eq::boundTheta} by noticing that for all $k\geq 2$, $\frac{k-1}{k} = 1 - \frac{1}{k} \geq \frac{1}{2}$.
		By using Lemma~\ref{lem:boundLambdaVariant3} once more in \eqref{eq::boundTheta}, it holds that,
		\begin{equation*}
			\lambda_k\theta_k = \frac{\lambda_k^2}{\lambda_{k-1}}  \geq
			\frac{1}{k^{1/2+\delta}} \frac{1}{2\sqrt{2}L} \frac{1}{2^{\delta}\sqrt{2}}\frac{\mu}{L} = \frac{1}{k^{1/2+\delta}}  \frac{\mu}{2^{\delta}4 L^2},
		\end{equation*}
		which is the left-hand side of \eqref{eq::boundLambdaTheta}. 
		
		Finally, for the right-hand side of \eqref{eq::boundLambdaTheta}, from \eqref{eq:choicestep}, 
		\begin{align*}
			\theta_k\lambda_k = \frac{\lambda_k^2}{\lambda_{k-1}} &\leq \lambda_{k-1}\left(1 + \left(1-\frac{1}{k^{1/2+\delta}}\right)\theta_{k-1}\right) = \lambda_{k-1}\left(1 + \theta_{k-1} -\frac{\theta_{k-1}}{k^{1/2+\delta}}\right)
			\\
			& = \lambda_{k-1}\left(1 + \theta_{k-1}\right) \left(1 - \frac{1}{k^{1/2+\delta}} \frac{\theta_{k-1}}{1+\theta_{k-1}}\right).
		\end{align*}
		One can easily show that the function $\theta\in\R\mapsto \frac{\theta}{1+\theta}$ is increasing on $\R$. So combining this with \eqref{eq::boundTheta} again,
		$$
		\frac{\theta_{k-1}}{1+\theta_{k-1}}  \geq \frac{\frac{1}{2^{\delta}\sqrt{2} }\frac{\mu}{L}}{1+\frac{1}{2^{\delta}\sqrt{2} }\frac{\mu}{L}}
		=
		\frac{\mu}{2^{\delta}\sqrt{2}L+\mu}.
		$$
		So 
		$$
		1 - \frac{1}{k^{1/2+\delta}} \frac{\theta_{k-1}}{1+\theta_{k-1}}\leq 1 -  \frac{1}{k^{1/2+\delta}} \frac{\mu}{2^{\delta}\sqrt{2}L+\mu},
		$$
		which proves the last bound.
	\end{proof}

	We can now prove Theorem~\ref{thm::CVrates}.
	
	\begin{proof}[Proof of Theorem~\ref{thm::CVrates}]
		As before, let us consider $(x_k)_{k\in\N}$ the iterates of Algorithm~\ref{algo::AdaSGD} with step-sizes $(\lambda_k)_{k\in\N}$ obtained with V-\textbf{III}. With the same arguments as above, upon restriction to a compact set, all $\gf[\xi]$ and $\gf$ are $L$-Lipschitz continuous due to Assumption~\ref{ass:VarianceControl}. In \ref{case::strcvx}, $f$ and the $f_\xi$'s are furthermore $\mu$-strongly convex.
		
		Under these stronger assumptions we can actually reproduce the steps of the analysis from Section~\ref{sec:lyap} but with sharper bounds. Indeed, we can go back to the convexity inequality \eqref{convexe} and replace it with
		\begin{multline*}
			- 2 \inner{\esp[k-1]{\lambda_k\gf[\xi_{k}](x_k)}}{x_k-x^\star} 
			= - 2 \lambda_k\inner{\gf(x_k)}{x_k-x^\star}
			\\ \leq - 2 \lambda_k\left(f(x_k)-f^\star\right) - \lambda_k\mu \norm{x_k-x^\star},
		\end{multline*}
		where we also dropped the expectations on $\lambda_k$ as our choice is not random conditionally on $\mathcal{F}_{k-1}$.
		We apply the same reasoning to \eqref{eq:(b)}: strong convexity implies that
		$$
		f(x_{k-1}) \geq f(x_{k}) - \inner{x_k-x_{k-1}}{\gf[](x_k)} + \frac{\mu}{2}\norm{x_k-x_{k-1}}^2
		$$
		and hence going back to \eqref{eq:(b)},
		\begin{align*}
			\begin{split}
				\esp[k-1]{B_k} &
				\leq 2 \frac{\lambda_k^2}{\lambda_{k-1}} (f(x_{k-1}) - f(x_{k}) ) -2\frac{\lambda_k^2}{\lambda_{k-1}} \mu \norm{x_k-x_{k-1}}^2
				\\
				& = 2 \frac{\lambda_k^2}{\lambda_{k-1}} (f(x_{k-1}) - f^\star) - 2\frac{\lambda_k^2}{\lambda_{k-1}} (f(x_{k}) - f^\star) -  \frac{\lambda_k^2}{\lambda_{k-1}} \mu \norm{x_k-x_{k-1}}^2.
			\end{split}
		\end{align*}
		We then proceed exactly as in Section~\ref{sec:lyap}, and eventually obtain a sharper bound on $\esp[k-1]{T_k}$:
		\begin{align}\label{eq:refinedB_k_2}
			\esp[k-1]{T_k}   \nonumber
			\leq & 
			\left(1 - \lambda_k \mu  \right)  \norm{x_{k}-x^\star}^2 + 2  \lambda_k \theta_{k}  \left(f(x_{k-1})-f^\star\right)  +   \left(1 - 2 \lambda_k \theta_{k}   \mu  \right)   \frac{\norm{\Delta_{k-1}}^2}{2}    \nonumber
			\\ &+ 4  \lambda_k^2\esp[k-1]{\norm{\gf[\xi_k](x_k) - \gf[\xi_{k-1}](x_k)}^2},
		\end{align}
		where we recall that $\Delta_{k-1} = \norm{x_k-x_{k-1}}$.
		
		We now use Lemma~\ref{lem:boundLambdaVariant3} and~\ref{lem::BoundsThetak}, more specifically, we use the left-hand side of \eqref{eq:boundLmu}, the right-hand side of \eqref{eq::boundLambdaTheta} and the left-hand side of \eqref{eq::boundLambdaTheta} on the first, second and third terms of \eqref{eq:refinedB_k_2}, respectively. This yields for all $k\geq 2$,
	\begin{multline*}
			\esp[k-1]{T_k}  
			\leq  
			\left(1 - \frac{1}{k^{1/2+\delta}}\frac{\mu}{2\sqrt{2}L}  \right)  \norm{x_{k}-x^\star}^2 
			+ 2  \left(1 + \theta_{k-1}\right) \left(1  - \frac{1}{k^{1/2+\delta}} \frac{\mu}{\mu+2^\delta\sqrt{2}L}\right) \left(f(x_{k-1})-f^\star\right)  
			\\ 
			+   \left(1 -  \frac{1}{k^{1/2+\delta}}\frac{\mu^2}{2^\delta2L^2} \right)   \frac{\norm{\Delta_{k-1}}^2}{2}    
			+ 4  \lambda_k^2\esp[k-1]{\norm{\gf[\xi_k](x_k) - \gf[\xi_{k-1}](x_k)}^2}.
	\end{multline*}
		Let $\tau = \min\left\{ \frac{\mu}{2\sqrt{2}L}, \frac{\mu}{\mu+2^\delta\sqrt{2}L}, \frac{\mu^2}{2^\delta2L^2} \right\}<1$, because $\mu/L\leq 1$. and so we obtain an improved version of Proposition~\ref{prop:lyap}: for all ${\color{black} k\geq 3}$,
		\begin{align*}
			\begin{split}
				\esp[k-1]{T_k}  
				\leq  \left(1 - \frac{\tau}{k^{1/2+\delta}}\right) {\color{black} T_{k-1}}+ 4  \lambda_k^2\esp[k-1]{\norm{\gf[\xi_k](x_k) - \gf[\xi_{k-1}](x_k)}^2},
			\end{split}
		\end{align*}
		
		It now remains to control the last term above, in the exact same way the proof of Theorem~\ref{thm:convergence}. We apply Lemma~\ref{lem::uniformControl} to get  
		{\color{black}
			\begin{align*}
				\begin{split}
					\esp[k-1]{T_k}  
					&\leq  \left(1 - \frac{\tau}{k^{1/2+\delta}}\right)T_{k-1} + 32 L  \lambda_k^2 \norm{x_k - x^\star}^2 + 16 \lambda_k^2\sigma^2,
					\\&\leq  \left(1 - \frac{\tau}{k^{1/2+\delta}}  + 32 L  \lambda_k^2\right)T_{k-1} +  16 \lambda_k^2\sigma^2.
				\end{split}
			\end{align*}
			
			Hence, using the upper bound \eqref{eq:boundLmu} for $\lambda_k$, letting $\gamma_k =   \frac{1}{\mu k^{1/2+\delta}}$, $\tilde{\mu} = \frac{\tau \mu}{2}$ and $\tilde{L} = 2 L$ and taking the expectation on both sides of the above inequality, we obtain that, for $k \geq 3$,
			$$
			\esp{T_k} \leq  \left(1 - 2 \tilde{\mu} \gamma_k  + 2 \tilde{L}  \gamma_k^2  \right) \esp{T_{k-1}} +  2 \sigma^2 \gamma_k^2.
			$$
			For $k \geq 0$, let us introduce the notation
			$$
			W_{k} = \esp{T_{k+2}}, \; \tilde{\gamma}_k = \gamma_{k+2}.
			$$
			Then, since $ \frac{1}{2 \mu k^{1/2+\delta}}  \leq \tilde{\gamma}_k \leq  \frac{1}{\mu k^{1/2+\delta}}$
			we thus obtain that, for $k \geq 1$,
			$$
			W_{k} \leq   \left(1 - 2 \frac{\tilde{\mu}}{2} \frac{1}{\mu k^{1/2+\delta}}  + 2 \tilde{L} \frac{1}{\mu^2  k^{1+2\delta}} \right) W_{k-1} +  2 \sigma^2 \frac{1}{\mu^2 k^{1+2\delta}}
			$$
			which corresponds to the deterministic recursion in  \cite{bach:hal-00608041}[Equation (16)]. Since $\tau < 1$ and $L / \mu \geq 1$, we have that $\frac{\tilde{\mu}}{2} \leq \tilde{L}$. Hence, arguing as in the proof    of   \cite{bach:hal-00608041}[Theorem 1], we obtain from  \cite{bach:hal-00608041}[Equation (5)] that, for all $k \geq 1$,
			$$
			W_{k} \leq  2 \exp \left( 2 \delta^{-1} \frac{\tilde{L}^2}{\mu^2} (1- k^{-2 \delta}) \right) \exp\left( - \frac{\tilde{\mu}}{8 \mu} k^{1/2-\delta}\right) \left( W_{0}+ \frac{\sigma^2}{\tilde{L}^2} \right) + \frac{8 \sigma^2}{\mu \tilde{\mu}}  k^{-(1/2 + \delta)}
			$$  
			for $0 < \delta < 1/2$.
			Hence, using that  $\tilde{\mu} = \frac{\tau \mu}{2}$ and $\tilde{L} = 2 L$ we finally obtain that, for all $k \geq 3$,
			\begin{multline*}
				\esp{T_k}  \leq 2 \exp \left( 
				8 \, \delta^{-1} \frac{L^2}{\mu^2} (1- (k-2)^{-2 \delta}) \right) \exp\left( 	- \frac{\tau}{16} (k-2)^{1/2-\delta}\right) \left( \esp{T_2}  + \frac{\sigma^2}{2 L^2} \right) 
				\\+ \frac{16 \sigma^2}{\tau \mu^2  } (k-2)^{-(1/2 + \delta)},
			\end{multline*}
			which completes the proof.
		}
	\end{proof}

	\section{The Convex Case \label{app:convex}}
	
	In this section, we consider a relaxation of the smooth and strongly convex case that we considered previously.
	Let us recall that, from \eqref{eq::Lyap} in Proposition~\ref{prop:lyap}, we have:
	\begin{alignat}{2} 
		\nonumber	&\esp[k-1]{\norm{x_{k+1}-x^\star}^2+\frac{\norm{x_{k+1}-x_k}^2}{2}} 
		&
		\\\nonumber \leq &
		\underbrace{\norm{x_{k}-x^\star}^2 +  \frac{\norm{x_{k}-x_{k-1}}^2}{2}}_{:=V_k}
		+ \underbrace{ 4  \esp[k-1]{\lambda_{k}^2}\esp[k]{\norm{\gf[\xi_{k}](x_{k}) - \gf[\xi_{k-1}](x_{k})}^2}}_{:=\beta_k} 
		&
		\\ & +
		\underbrace{2\frac{\esp[k-1]{\lambda_k^2}}{\lambda_{k-1}}\left(f(x_{k-1})-f^\star\right) 
			-
			2\left(\esp[k-1]{\lambda_k}+\frac{\esp[k-1]{\lambda_k^2}}{\lambda_{k-1}}\right)\left(f(x_k)-f^\star\right)}_{:=-U_k} 
		& \label{eq_notations}
	\end{alignat}

	From now on, we assume that the step-size $\lambda_k$ satisfies the following update rule:
	
	\begin{equation} \label{eq:choicestepf}
		\lambda_k  = \min 
		\begin{array}{l}
			\left\{ \frac{\norm{x_k-x_{k-1}}}{2\sqrt{2}\norm{\gf[\xi_{k-1}](x_k) - \gf[\xi_{k-1}](x_{k-1})}},   \lambda_{k-1} \sqrt{ 1 + \theta_{k-1} },
			\right.
			\left.
			\frac{f(x_k)-f^\star}{f(x_{k-1})-f(x_k)} \lambda_{k-1} \mathbf{1}_{\{f(x_{k-1})-f(x_k)>0\}}
			\right\}.
		\end{array}
	\end{equation}
	
	Remark that with choice \eqref{eq:choicestep}, $\lambda_k$ is not a random variable conditionally on $\mathcal{F}_{k-1}$, so we simply have
	$\esp[k-1]{\lambda_k}=\lambda_k$ and $\esp[k-1]{\lambda_k^2} = \lambda_k^2$ in \eqref{eq_notations}.

	\paragraph{Remark:}
	One should notice that the step-size rule is different in this last part of the paper, with an additional third condition to ensure the non-negativeness of $U_k$.
	As for the second condition, we chose the simplest one, coming from \citet{malitsky2019adaptive}, but we could of course use stronger conditions such as in previous subsections.
	
	Since we do no longer make a strong convexity assumption, we will need the next one instead. 
	\begin{assumption}[Square Summability of the step-sizes] \label{hyp:summabilitybis}
		
		The sequence $(\lambda_{k})_{k \in \N}$ of step-sizes satisfies
		\begin{equation*}
			\sum_{k \geq 0} {\lambda_{k}^2} < +\infty.
		\end{equation*}
	\end{assumption}

	\begin{lemma} \label{prop_conv}
		Suppose that Assumption \ref{ass:VarianceControl} 
		holds, and that Assumption \ref{hyp:summabilitybis} is satisfied for the choice of step-size \eqref{eq:choicestepf}.
		Then, using the notations of Equation (\ref{eq_notations}) we have that
		
		\begin{equation*}
			\esp[k-1]{V_k} \leq V_{k-1} -U_k + \beta_k
		\end{equation*}
		with $\sum_{k\geq 0} \beta_k < + \infty$, and $U_k \geq 0$.
		
	\end{lemma}

	\begin{proof}

		Thanks to Assumption~\ref{ass:VarianceControl} and the hypothesis that  $\lambda_k$ satisfies Assumption  \ref{hyp:summabilitybis}, we have that
		$\sum_{k\geq 0} \beta_k < + \infty$. Now let us consider $U_k$ to check when it is non-negative:
		
		\begin{equation*}
			U_k =-2\frac{{\lambda_k^2}}{\lambda_{k-1}}\left(f(x_{k-1})-f^\star\right) 
			+
			2\left({\lambda_k}+\frac{{\lambda_k^2}}{\lambda_{k-1}}\right)\left(f(x_k)-f^\star\right),
		\end{equation*}
		that is
		\begin{equation*}
			U_k =2\frac{{\lambda_k^2}}{\lambda_{k-1}}\left(f(x_k)-f(x_{k-1})\right) 
			+
			2{\lambda_k}\left(f(x_k)-f^\star\right)
		\end{equation*}
		Notice that we always have
		$f(x_k)-f^\star \geq 0$. Two cases may occur:
		\begin{enumerate}
			\item[(a)] If $f(x_k) \geq f(x_{k-1})$, then we always have $U_k \geq 0$.
			
			\item[(b)] If $f(x_{k-1}) > f(x_k)$, then we see that $U_k \geq 0$ if and only if
			\begin{equation*}
				\lambda_k \leq \frac{f(x_k)- f^\star}{f(x_{k-1})- f(x_k)} \lambda_{k-1},
			\end{equation*}

		\end{enumerate}
		which concludes the proof thanks to the choice \eqref{eq:choicestepf} for $\lambda_k$.

	\end{proof}

	\begin{proposition}
		Suppose that Assumptions \ref{ass:VarianceControl} holds, and that Assumption \ref{hyp:summabilitybis} is satisfied for the choice of step-size \eqref{eq:choicestepf}. 
		Then, we have that
		\begin{equation*}
			\sup_{k\geq 0} \esp{\|x_k - x^*\|} < +\infty
		\end{equation*}
	\end{proposition}
	
	\begin{proof}
		Thanks to Lemma~\ref{prop_conv}, we can apply the Robbins-Siegmund theorem (Theorem~\ref{th:RS}) with $V_k=\norm{x_{k}-x^\star}^2 +  \frac{\norm{x_{k}-x_{k-1}}^2}{2}$ which concludes the proof.
	\end{proof}

	We can now give convergence results.

	\begin{proposition} \label{prop_convex}
		
		Let us define
		\begin{equation*}
			\hat{x}_K=\frac{ \left( {\lambda_K}+\frac{{\lambda_K^2}}{\lambda_{K-1}} \right) x_K + \sum_{k=1}^{K-1} w_k x_k}{S_K}
		\end{equation*}
		with
		\begin{equation*}
			w_k = {\lambda_k}+\frac{{\lambda_k^2}}{\lambda_{k-1}} - \frac{{\lambda_{k+1}^2}}{\lambda_{k}}
		\end{equation*}
		and 
		\begin{equation*}
			S_K =  {\lambda_K}+\frac{{\lambda_K^2}}{\lambda_{K-1}} + \sum_{k=1}^{K-1}w_k
			=\frac{{\lambda_1^2}}{\lambda_{0}} + \sum_{k=1}^{K} {\lambda_k}
		\end{equation*}
		Suppose that Assumption \ref{ass:VarianceControl} holds, and that Assumption \ref{hyp:summabilitybis} is satisfied for the choice of step-size \eqref{eq:choicestepf}.
		Then, there exists $C>0$ such that 
		we have
		\begin{equation*}
			f(\hat{x}_K)-f^* \leq \frac{C}{S_K} \mbox{ a.s.}
		\end{equation*}
	\end{proposition}

	\begin{proof}
		
		We have:
		\begin{eqnarray*}
			\sum_{k=1}^K U_k & = & 2\left({\lambda_K}+\frac{{\lambda_K^2}}{\lambda_{K-1}}\right)\left(f(x_K)-f^\star\right)
			-2 \frac{\esp[0]{\lambda_1^2}}{\lambda_{0}}\left(f(x_{0})-f^\star\right) \\
			& &
			+ 2 \sum_{k=1}^{K-1} \left(
			\underbrace{\left({\lambda_k}+\frac{{\lambda_k^2}}{\lambda_{k-1}} - \frac{{\lambda_{k+1}^2}}{\lambda_{k}}\right)}_{=w_k}
			\left(f(x_k)-f^\star\right) \right)
		\end{eqnarray*}
		Thanks to Lemma~\ref{prop_conv}, we can apply  Robbins-Siegmund theorem (Theorem~\ref{th:RS}) to deduce that  $\sum_{k\geq 0} U_k < +\infty$ a.s. Then, notice that from the update rule \eqref{eq:choicestepf} of $\lambda_k$, we have that $w_k \geq 0$.
		We thus deduce that there exists a constant $C>0$ such that:
		\begin{equation*}
			\left({\lambda_K}+\frac{{\lambda_K^2}}{\lambda_{K-1}}\right)\left(f(x_K)-f^\star\right)
			+
			\sum_{k=1}^{K-1}
			w_k
			\left(f(x_k)-f^\star\right)
			\leq C
			\mbox{ a.s.}
		\end{equation*}
		Since $f$ is convex, we now want to use Jensen inequality, that states that if $\sum_{k=1}^K \rho_k =1$:
		\begin{equation*}
			f\left( \sum_{k=1}^K \rho_k x_k \right) \leq \sum_{k=1}^K \rho_k f(x_k)
		\end{equation*}

		So that we get
		\begin{equation} \label{eq_conv}
			f(\hat{x}_K)-f^* \leq \frac{C}{S_K} \mbox{ a.s.}
		\end{equation}
		which concludes the proof.
	\end{proof}

	We finally add a last assumption.
	\begin{assumption}[Divergence of the steps] \label{hyp:divergence}
		The random sequence $(\lambda_{k})_{k \in \N}$ of step-sizes satisfies
		
		\begin{equation*}
			\sum_{k \geq 0} {\lambda_{k}} = + \infty 
		\end{equation*}
	\end{assumption}
	
	Notice that  Assumption \ref{hyp:divergence} is satisfied  in the smooth and strongly convex case.
	Notice also that this assumption would be directly satisfied if we only had the two first conditions in the update rule of $\lambda_k$.
	
	\begin{corollary}
		
		Suppose that Assumptions \ref{ass:VarianceControl} holds, and that Assumption \ref{hyp:summabilitybis} and \ref{hyp:divergence}  are satisfied for the choice of step-size \eqref{eq:choicestepf}.
		Then, we have that
		\begin{equation*}
			f(\hat{x}_K)-f^* \to 0 \mbox{ a.s.}
		\end{equation*}
	\end{corollary}
	
	\begin{proof}
		Notice that 
		\begin{equation*}
			S_K = 
			\frac{{\lambda_1^2}}{\lambda_{0}} + \sum_{k=1}^{K} {\lambda_k} \geq \sum_{k=1}^{K} {\lambda_k} 
		\end{equation*}
		Hence, by using Assumption \ref{hyp:divergence} in \eqref{eq_conv}, we get the result of the corollary.
		
	\end{proof}
	%


	\section{Details on the Experiments\label{app:expdetails}}

	\begin{figure}[t]
		\centering
		\begin{tabular}{ccc}
			\multicolumn{3}{c}{\includegraphics[width=0.8\textwidth]{Figures/sensitivity_legend.png}}
			\\
			\makecell{
				\footnotesize Linear -- Synthetic\\
				\includegraphics[width=0.25\textwidth]{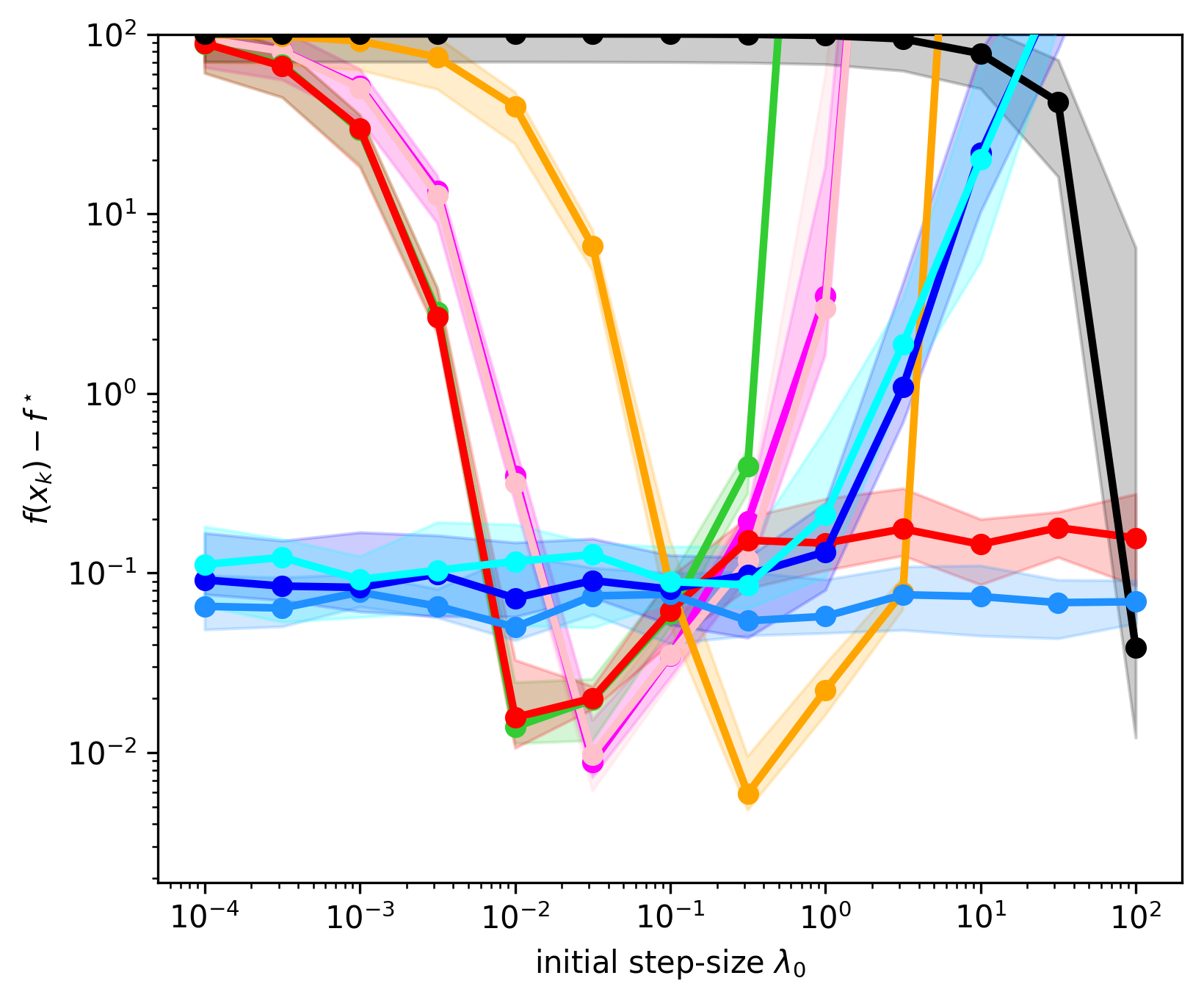}
			}&
			\makecell{
				\footnotesize Sum-of-Ridges -- Synthetic\\
				\includegraphics[width=0.25\textwidth]{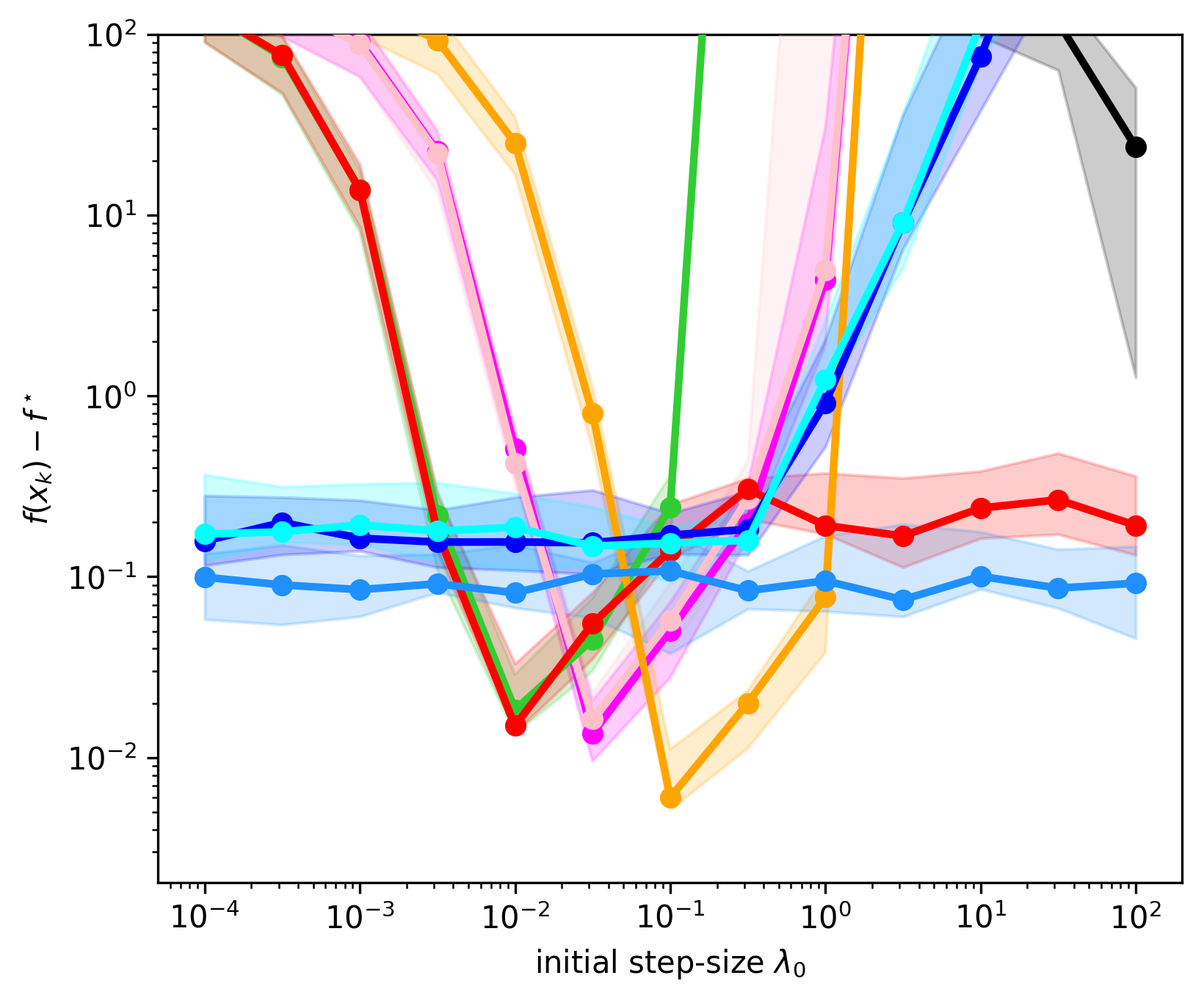}
			}&
			\makecell{
				\footnotesize Logistic -- 2moons\\
				\includegraphics[width=0.25\textwidth]{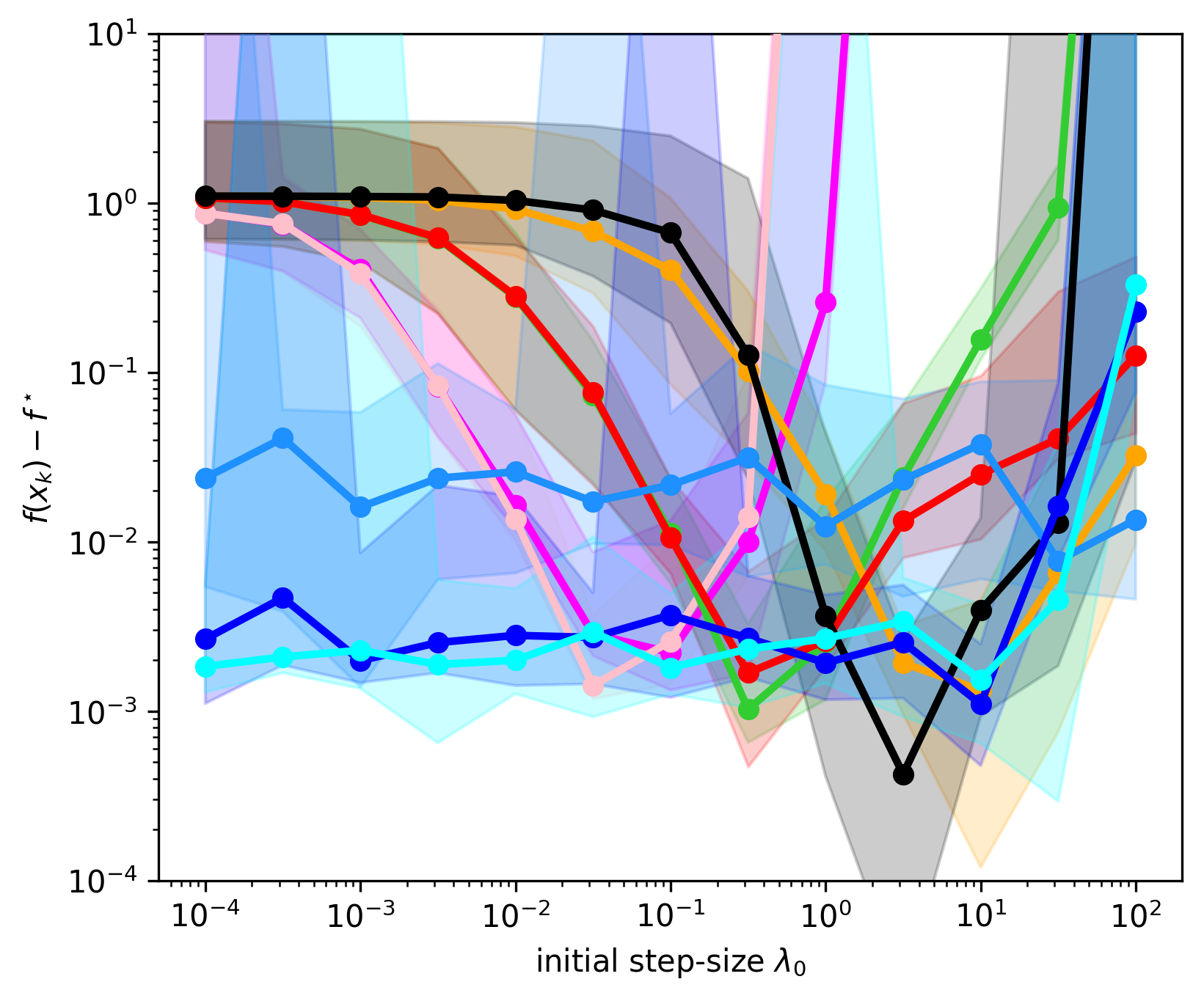} 
			}
			\\
			\makecell{
				\includegraphics[width=0.25\textwidth]{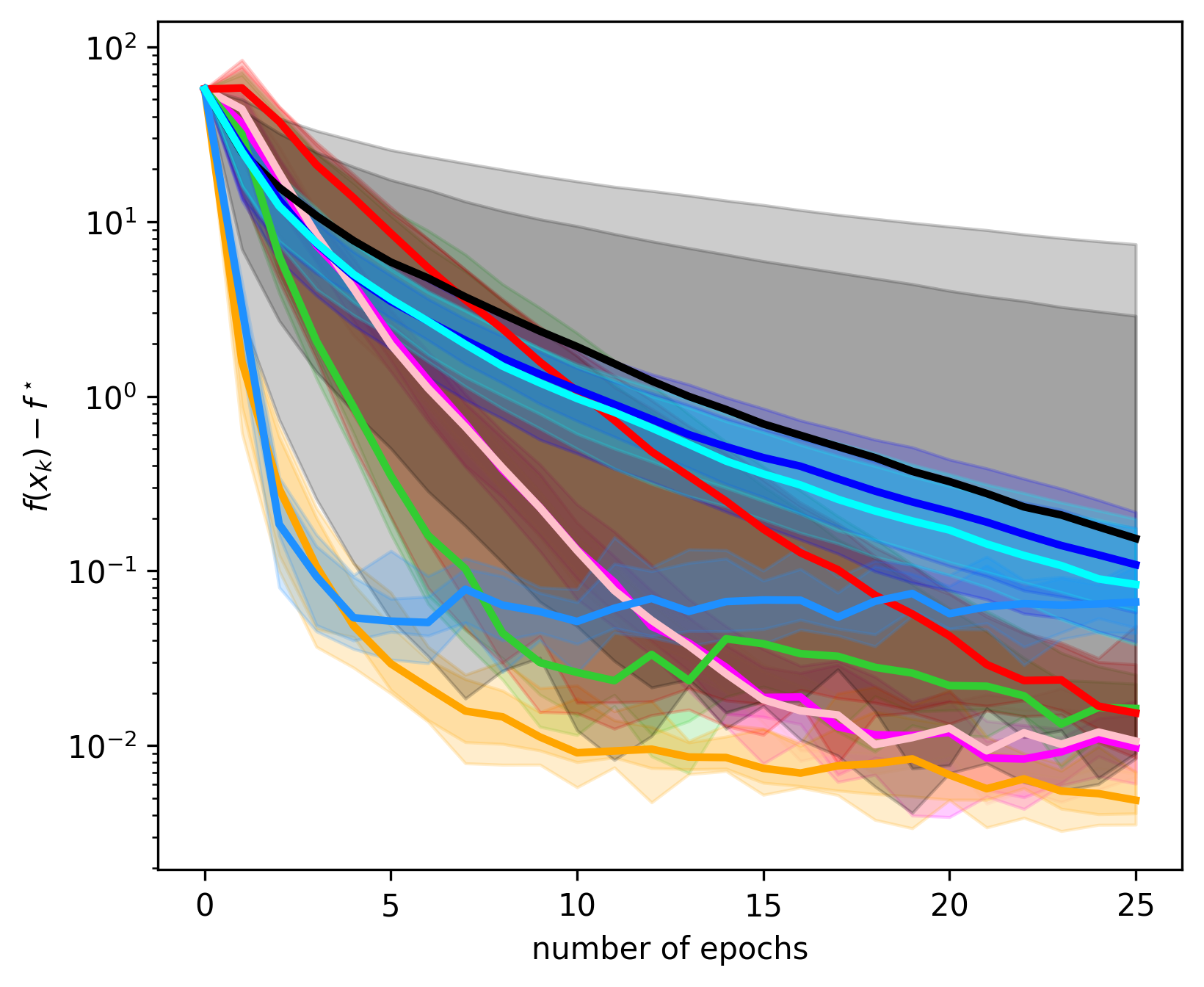}
			}&
			\makecell{
				\includegraphics[width=0.25\textwidth]{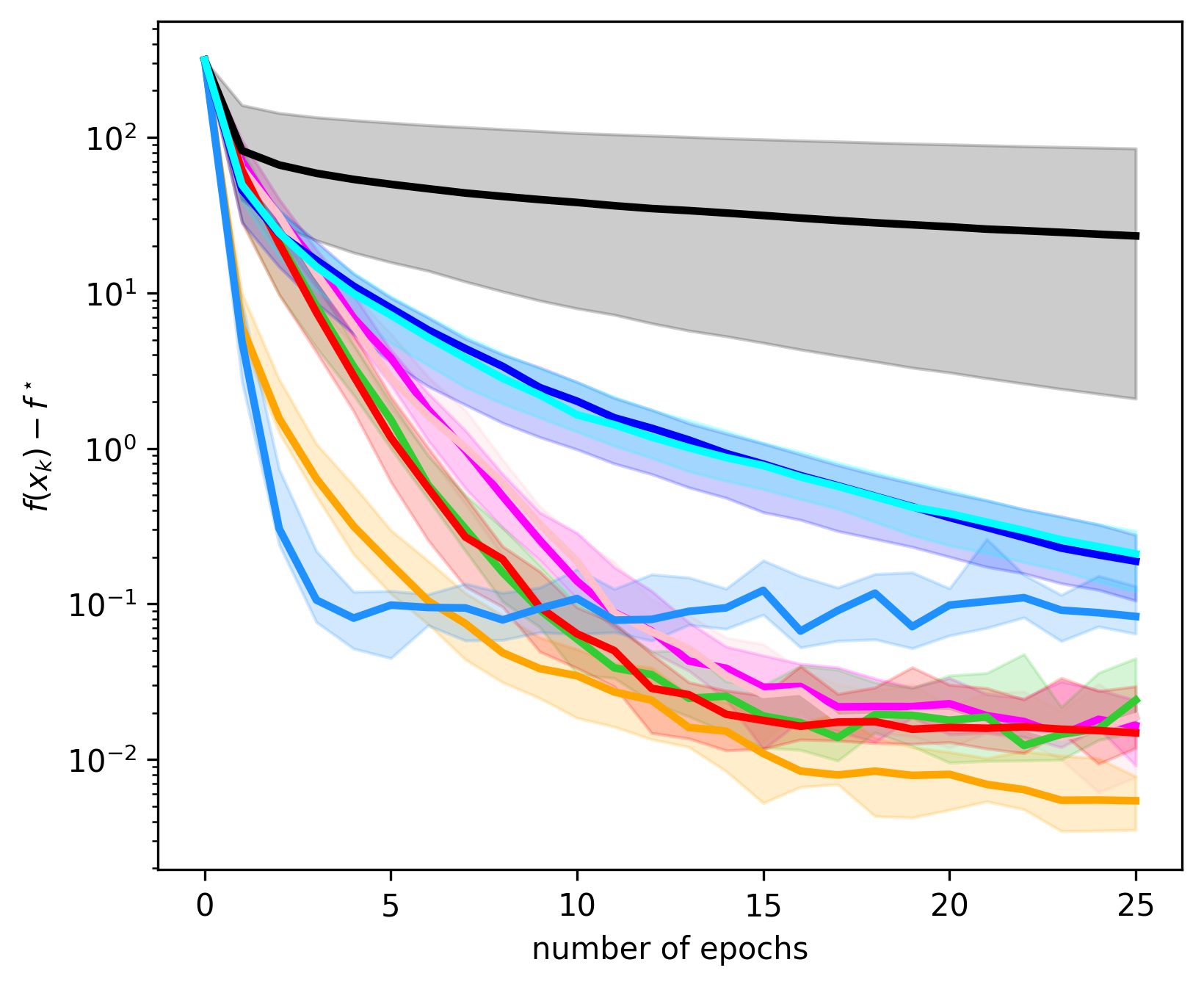} 
			}&
			\makecell{
				\includegraphics[width=0.25\textwidth]{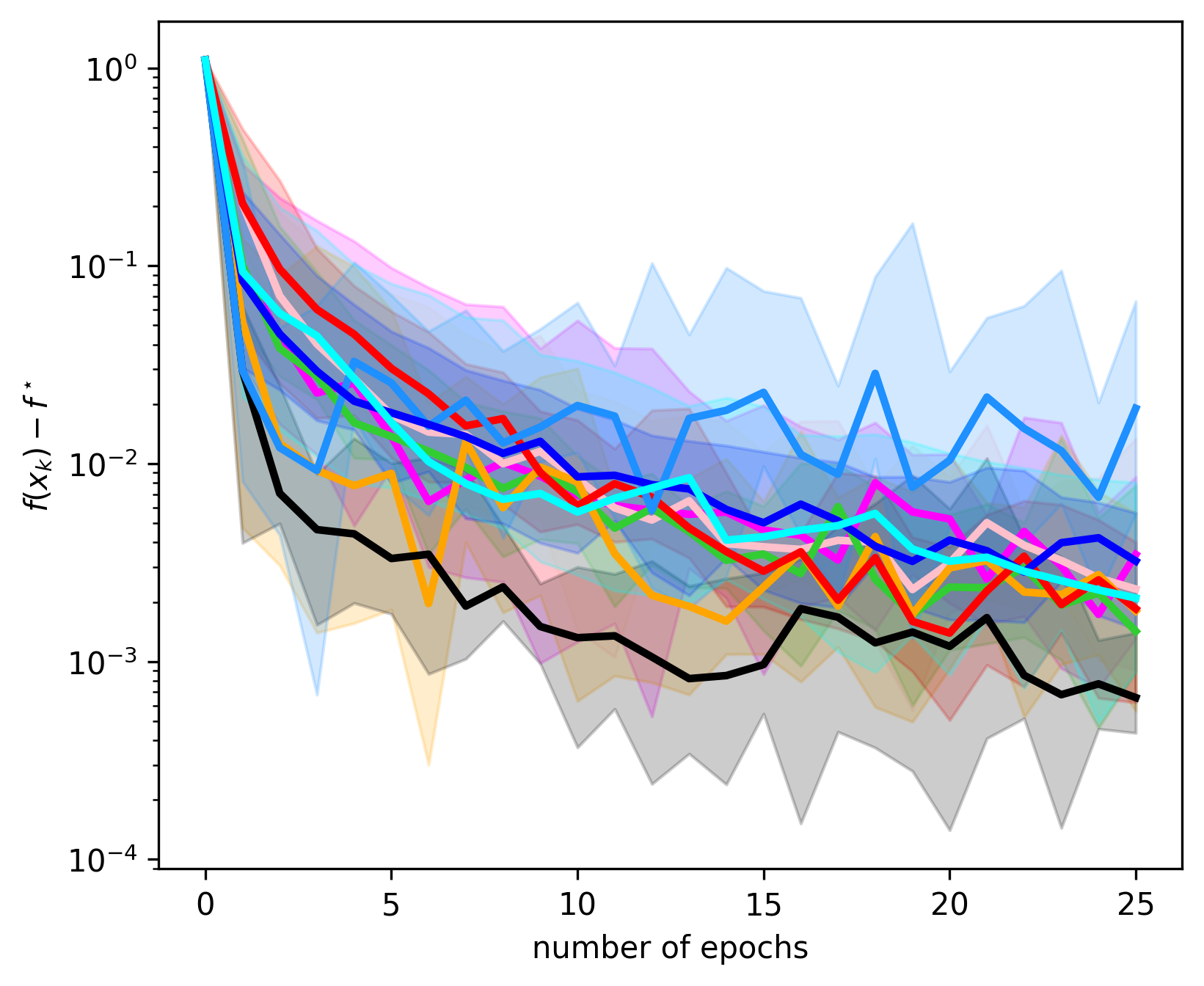} 
			}
			\\
		\end{tabular}
		\caption{Sensitivity to $\lambda_0$ and minimization performance after tuning for \textbf{smaller} mini-batches (of size $8$).
			Same experimental setting as in Figures~\ref{fig::stepsize_sensitivity}, \ref{fig::detailed_sensitivity} and~\ref{fig::EXP_loss}.\label{fig::smallerMB}}
	\end{figure}

	\paragraph{Objective functions.}
	We performed numerical experiments on problems of the form \eqref{eq:optim_fullbatch} \ie where $f = \frac{1}{N}\sum_{\ell = 1}^{N}  f_{\ell}$. We considered four types of loss functions defined respectively for all $x\in\R^d$ by:
	\begin{itemize}
		\item Linear regression: $f_{\ell}(x) = \frac{1}{2}(y_{\ell} - w_{\ell}^Tx)^2$
		\item Ridge regression: $f_{\ell}(x) = g_{\ell} \left(y_{\ell} - w_{\ell}^T x\right) $ where $g_{\ell}(t) = \frac{t^4}{1+t^2} + 10^{-2} t^2$ is $10^{-2}$-strongly convex.
		\item Logistic regression: $f_{\ell}(x) = \log\left( 1 + \exp(-y_{\ell} w_{\ell}^T x )  \right)$ 
		\item Poisson regression: $f_{\ell}(x) =  \log\left( -y_{\ell} w_{\ell}^T x   +  1 + \exp( w_{\ell}^T x ) \right)$
	\end{itemize}
	In the above, the dataset  is made of couples $(w_{\ell},y_{\ell}) \in \R^{d} \times \R$ for $\ell \in \{1, \ldots, N\}$. Below we describe the data as a matrix $W\in\R^{N\times d}$ and a vector $y\in\R^N$.

	\paragraph{Datasets}
	Based on the above, we used synthetic and real-world datasets of the form $(W,y)$ with $W\in\R^{N\times d}$ and $y\in\R^N$, where $N$ is the number of data samples.
	\begin{itemize}
		\item Synthetic datasets were generated by drawing data at random. Each element of $W$ is i.i.d.\ $\mathcal{N}(0, 1)$. We do the same for the vector $y$. We used $N=200$ and $d=20$. We used batch-size of 32 (\ie $\mathrm{Card}(\xi_k)=32$).
		\item The \textit{2moons} dataset is available via the scikit learn library \citep{scikit-learn}. It is made of $N=200$  samples in dimension $d=3$. We used a batch-size of $32$.
		\item The \textit{w8a} dataset\footnote{\url{https://www.csie.ntu.edu.tw/~cjlin/libsvmtools/datasets/binary.html}} is a publicly available dataset classically used for logistic regression. It has a significantly larger number $N=49749$ of samples, and is also larger-dimensional as $d=300$. We used a batchsize of $309$.
	\end{itemize}
	All the algorithms described below are ran for the same number of epochs, where one epoch consists in using $N$ data samples. That is, one epoch corresponds to $\frac{N}{\mathrm{batch-size}}$ iterations. At each iteration, we sample the mini-batches of data $(\xi_k)_{k\in\N}$ uniformly at random with replacement.
	
	\paragraph{Algorithms}
	\begin{itemize}
		\item SGD is the algorithm described in \eqref{eq:algosto}. It is either used with no decay ($\lambda_k=\lambda_0$) or decay $\lambda_k = \frac{\lambda_0}{k^{1/2+\delta}}$.
		\item We implemented the three variants of AdaSGD exactly as described in Algorithm~\ref{algo::AdaSGD} and \eqref{eq:choicestep}. 
		\item AdaSGD-MM is implemented as proposed by \citet{malitsky2019adaptive}. As previously discussed, it is SGD with a step-size that is computed for $k\geq 2$ as $
		\lambda_k  = \min \left\{ \alpha \Lambda_k,   \lambda_{k-1} \sqrt{ 1 + \theta_{k-1} } \right\}
		$
		with either
		$
		\Lambda_k =  \frac{\norm{x_k-x_{k-1}}}{\norm{\gf[\xi_k](x_k) - \gf[\xi_k](x_{k-1})}}$ (referred to as biased)
		or
		$
		\Lambda_k =  \frac{\norm{x_k-x_{k-1}}}{\norm{\gf[\zeta_k](x_k) - \gf[\zeta_k](x_{k-1})}}$ (unbiased), where $\zeta_k$ is an independent copy of $\xi_k$,
		and where $\alpha > 0$. To ease the discussion and comparison, in Section~\ref{sec:num}, we call $\lambda_0$ this parameter $\alpha$ as it is the tunable parameter of the method.
		\item Step-size Adagrad uses the step-size $\lambda_k = \frac{\lambda_0}{\sqrt{v_k}}$ where $v_k$ is updated iteratively: $v_{k} = v_{k-1} + \norm{\nabla f_{\xi_k}(x_k)}^2$, starting with $v_{-1}=0$.
		\item The stochastic Polyak-stepsize is defined by $\lambda_k =\min\left(
		\frac{f_{\xi_{k}}(x_k) - f_{\xi_{k}}^\star}{\norm{\gf[\xi_k](x_k)}^2}
		, \lambda_0\right)$.
		It is important to note that this step-size rule additionally requires evaluating $f_{\xi_{k}}(x_k)$ at each iteration and is  far from being tuning free. Indeed, it features a step-size $\lambda_0$ that needs to be tuned but also assumes knowledge of all the optimal values $f_\ell$, $n\in\{1,\ldots, N\}$. For fair comparisons with other methods, here we replace the optimal value $f_\ell^\star$ by $0$ for all $n$, since one cannot assume to know these values \textit{a priori}.
	\end{itemize}

	\begin{figure}[t]
		\centering
		\begin{tabular}{ccc}
			\multicolumn{3}{c}{\includegraphics[width=0.8\textwidth]{Figures/sensitivity_legend.png}}
			\\
			\makecell{
				\footnotesize Linear -- Synthetic\\
				\includegraphics[width=0.25\textwidth]{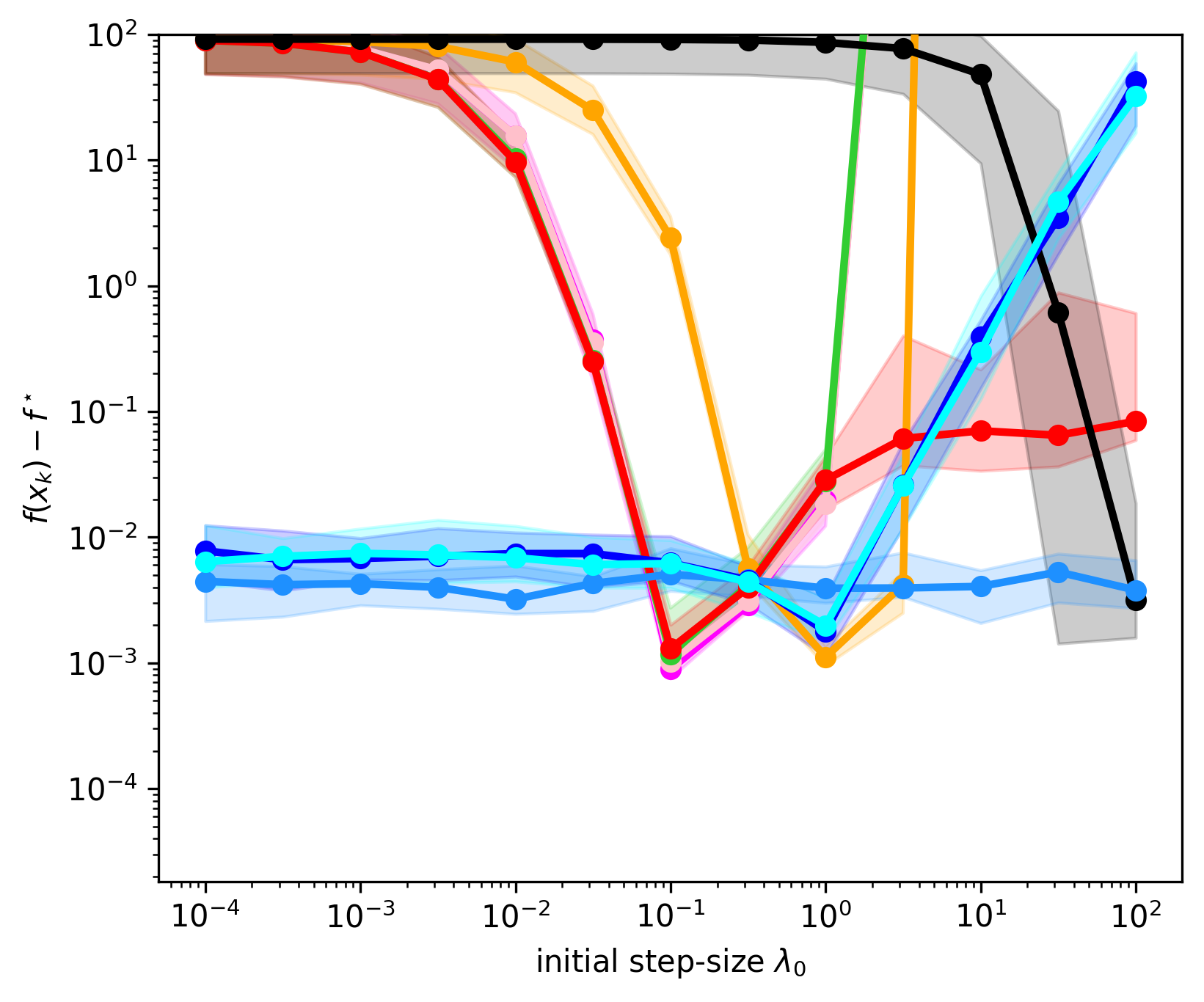}
			}&
			\makecell{
				\footnotesize Sum-of-Ridges -- Synthetic\\
				\includegraphics[width=0.25\textwidth]{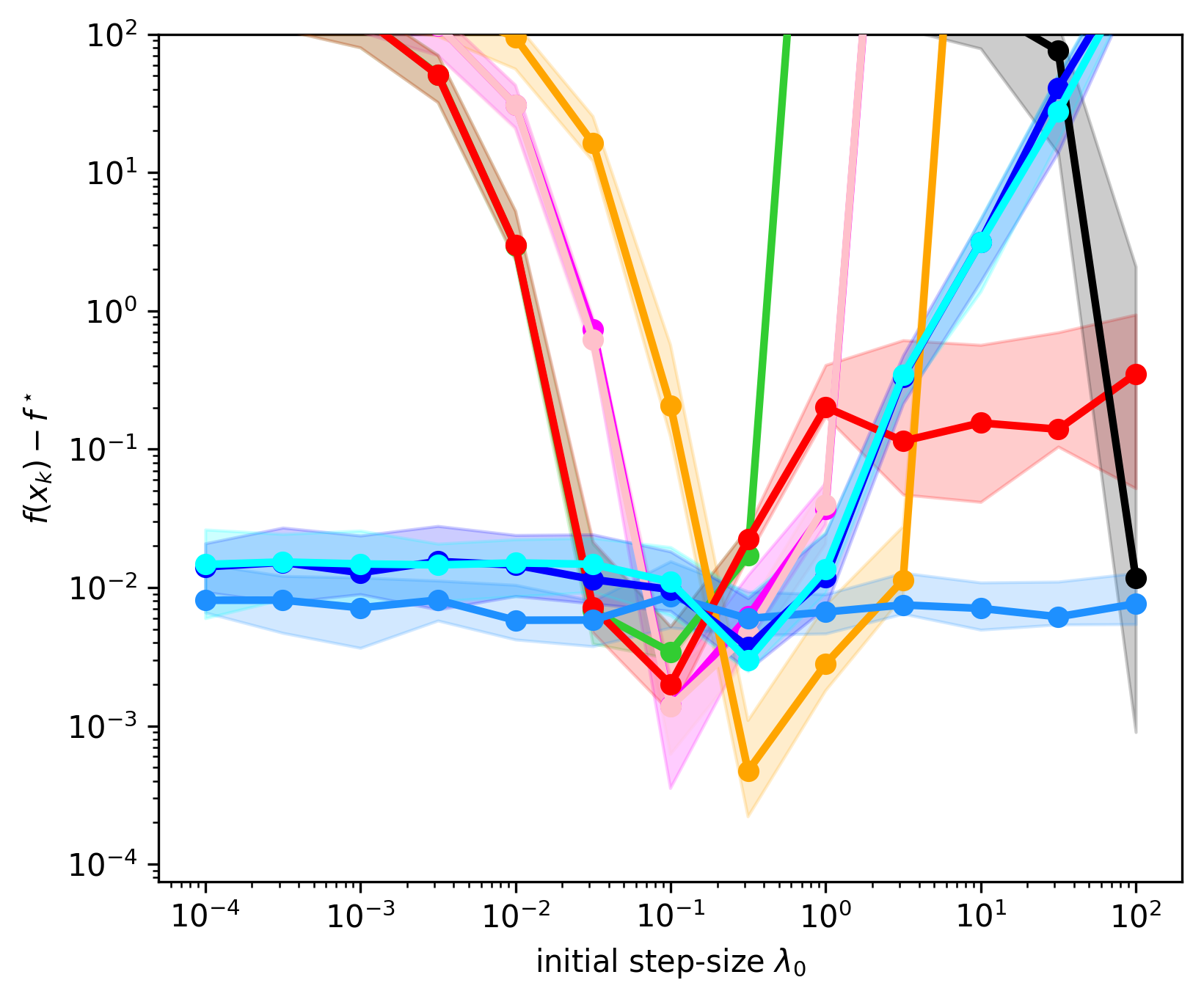} 
			}&
			\makecell{
				\footnotesize Logistic -- 2moons\\
				\includegraphics[width=0.25\textwidth]{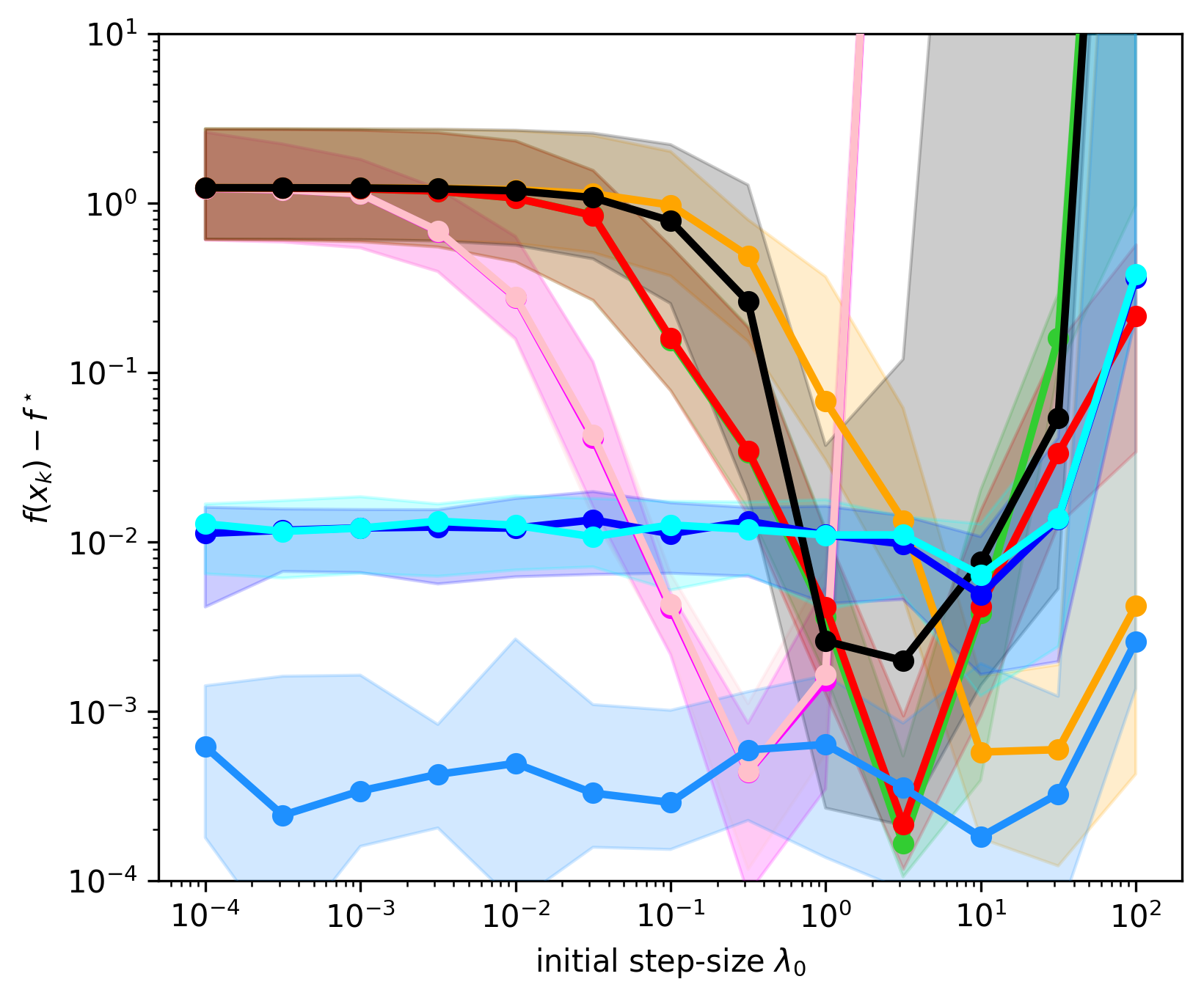} 
			}
			\\
			\makecell{
				\includegraphics[width=0.25\textwidth]{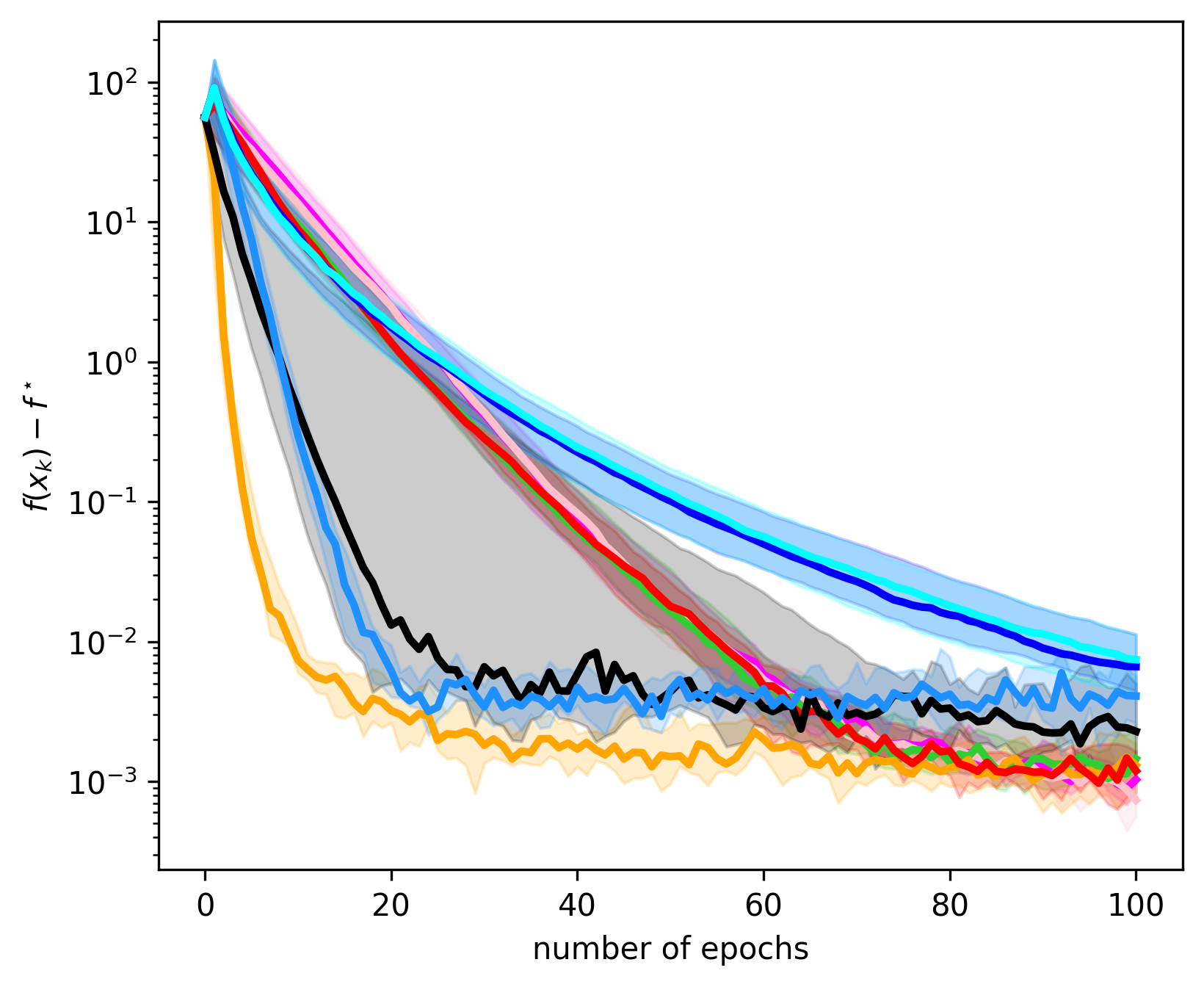}
			}&
			\makecell{
				\includegraphics[width=0.25\textwidth]{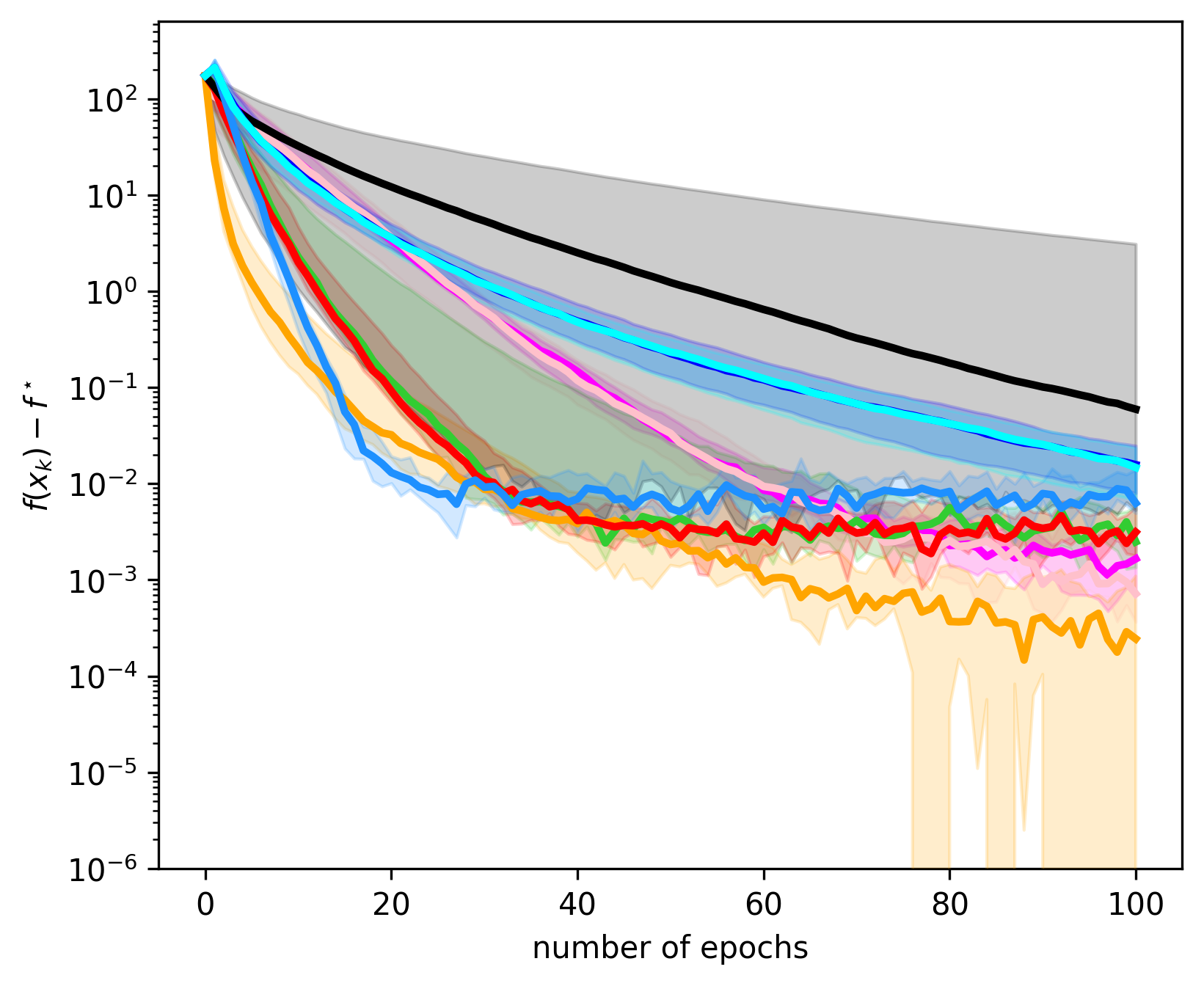} 
			}&
			\makecell{
				\includegraphics[width=0.25\textwidth]{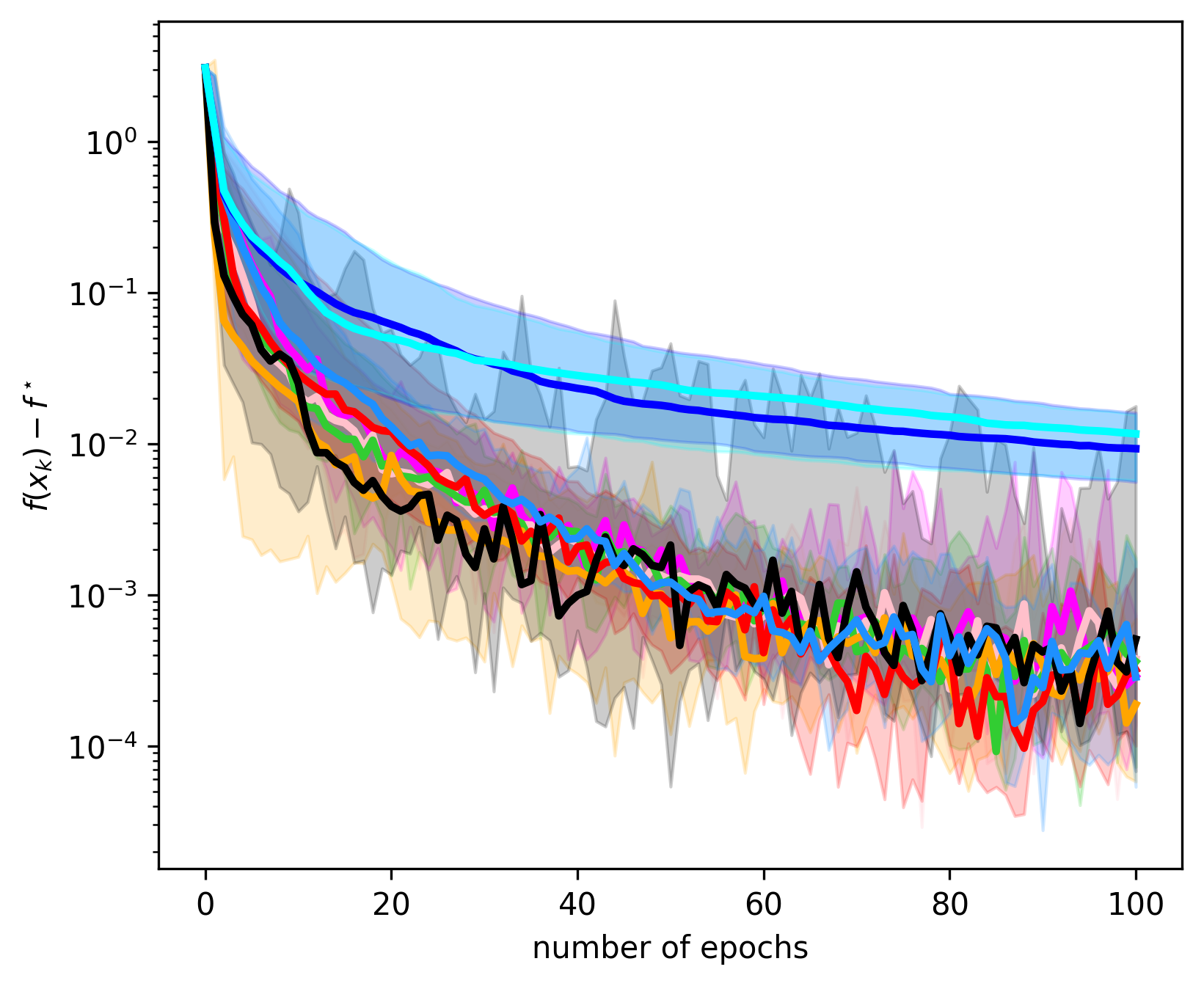} 
			}
			\\
		\end{tabular}
		\caption{Sensitivity to $\lambda_0$ and minimization performance after tuning for \textbf{larger} mini-batches (of size $128$).
			Same experimental setting as in Figures~\ref{fig::stepsize_sensitivity}, \ref{fig::detailed_sensitivity} and~\ref{fig::EXP_loss}.}\label{fig::largerMB}
	\end{figure}
	\paragraph{Parameter selection.}	
	\begin{itemize}
		\item All the methods feature a single tunable parameter $\lambda_0$. To obtain Figures~\ref{fig::stepsize_sensitivity} and~\ref{fig::detailed_sensitivity}, we ran the algorithms described above for all choices $\lambda_0$ on a dense grid of values. More specifically, we tried every $\lambda_0$ of the form $10^i$, for $i\in\{-4, -3.5,\ldots, 1.5, 2\}$. We ran $100$ epochs for each problem, except for the \textit{w8a} dataset for which we performed the grid-search on $10$ epochs due to a significantly higher computational cost.
		\item For Figure~\ref{fig::EXP_loss}, and bottom-rows of Figures~\ref{fig::smallerMB} and~\ref{fig::largerMB}, we selected the parameter $\lambda_0$ that achieved the lowest full-batch value $f(x_k)$ at the last epoch of the grid-search. We then ran again the algorithm with the best choice of $\lambda_0$ observed in the above experiments. We proceeded this way for all the baselines (SGD and AdaSGD-MM), however, since we argue that AdaSGD requires no tuning, we did not select $\lambda_0$ for our methods (AdaSGD), and rather always used a step-size $\lambda_0=10^{-3}$.
		\item For SGD with decay, as well as V-\textbf{II} and V-\textbf{III} of AdaSGD, we use $\delta=10^{-2}$. This value was \emph{not tuned} and simply taken small as any small $\delta>0$ works from a theoretical point of view.
	\end{itemize}

	\subsubsection*{Acknowledgments}
	The authors gratefully acknowledge financial support from the Agence Nationale de la Recherche  (MaSDOL grant ANR-19-CE23-0017), and by PEPR PDE-AI.
	
	
	\bibliographystyle{abbrvnat}
	\bibliography{biblio}

	\end{document}